\documentclass[a4paper,fleqn]{cas-sc}

\usepackage{amsmath,amsfonts}
\usepackage{algorithmic}
\usepackage{algorithm}
\usepackage{array}
\usepackage[caption=false,font=normalsize,labelfont=sf,textfont=sf]{subfig}
\usepackage{textcomp}
\usepackage{stfloats}
\usepackage{url}
\usepackage{verbatim}
\usepackage{graphicx}

\usepackage{cite}
\usepackage{url}
\usepackage{amsthm}
\newtheorem{theorem}{Theorem}
\usepackage{amsmath}
\usepackage{graphicx}
\usepackage{booktabs}
\newtheorem{definition}{Definition}
\usepackage{xcolor}
\usepackage{wrapfig}
\definecolor{cvprblue}{rgb}{0.21,0.49,0.74}
\definecolor{electricindigo}{rgb}{0.44,0,1}
\usepackage{multirow}
\definecolor{deblue}{RGB}{11,132,147}
\definecolor{ocra}{RGB}{204, 119, 34}
\usepackage{enumitem}
\usepackage{tikz}
\newcommand{\fcircle}[2][red,fill=red]{\tikz[baseline=-0.5ex]\draw[#1,radius=#2] (0,0.03) circle ;}

\definecolor{deblue}{RGB}{11,132,147}
\definecolor{ocra}{RGB}{204, 119, 34}
\definecolor{electricindigo}{rgb}{0.44, 0.0, 1.0}
\usepackage[capitalize]{cleveref}
\crefname{section}{Sec.}{Secs.}
\Crefname{section}{Section}{Sections}
\Crefname{table}{Table}{Tables}
\crefname{table}{Tab.}{Tabs.}

\definecolor{indigo(web)}{rgb}{0.29, 0.0, 0.51}
\usepackage{multirow}
\usepackage{xcolor}
\definecolor{darkblue}{RGB}{40,40,85}
\definecolor{babyblue}{rgb}{0.54, 0.81, 0.94}
\definecolor{pearDark}{HTML}{2980B9}
\definecolor{pearDarker}{HTML}{1D2DEC}
\usepackage[capitalize]{cleveref}
\crefname{section}{Sec.}{Secs.}
\Crefname{section}{Section}{Sections}
\Crefname{table}{Table}{Tables}
\crefname{table}{Tab.}{Tabs.}
\usepackage{fontawesome5}
\newtheorem{proposition}[theorem]{Proposition}
\usepackage{multirow}

\usepackage{colortbl}
\usepackage{tabulary}
\usepackage{etoolbox}
\usepackage{tikz}
\usepackage{pifont}
\makeatletter
\@namedef{ver@everyshi.sty}{}
\makeatother
\usepackage{tikz}
\usepackage{pifont}
\newcommand{\Attention}{\operatorname{Attention}}
\newcommand{\softmax}{\operatorname{softmax}}
\def\tsc#1{\csdef{#1}{\textsc{\lowercase{#1}}\xspace}}
\tsc{WGM}
\tsc{QE}

\begin{document}
\let\WriteBookmarks\relax
\def\floatpagepagefraction{1}
\def\textpagefraction{.001}

\shorttitle{}    

\shortauthors{}  

\title [mode = title]{Mamba Neural Operator: Who Wins? Transformers vs. State-Space Models for PDEs}  



%

\author[1]{Chun-Wun~Cheng}\fnref{fn1}
\affiliation[1]{organization={Department of Applied Mathematics and Theoretical Physics, University of Cambridge},
            addressline={The Old Schools, Trinity Lane}, 
            city={Cambridge},
            postcode={CB2 1TN}, 
            country={United Kingdom}}
\author[2]{Jiahao~Huang}\fnref{fn1}
\affiliation[2]{organization={Bioengineering Department and Imperial-X, Imperial College London},
            city={London},
            postcode={W12 7SL}, 
            country={United Kingdom}}
\author[3]{Yi~Zhang}
\affiliation[3]{organization={Department of Electronics, Southern University of Science and Technology},
            addressline={No. 1088 Xueyuan Avenue, Nanshan District}, 
            city={Shenzhen, Guangdong},
            postcode={518055}, 
            country={China}}
\author[2,4]{Guang~Yang}
\affiliation[4]{organization={School of Biomedical Engineering and Imaging Sciences, King's College London},
            addressline={WC2R 2LS London}, 
            country={United Kingdom}}
\author[1]{Carola-Bibiane~Sch\"onlieb}
\author[5]{Angelica~I.~Aviles-Rivero}\corref{cor1}
\affiliation[5]{organization={Yau Mathematical Sciences Cente, Tsinghua University },
            addressline={Jingzhai, }, 
            city={Haidian  District, Beijing},
            postcode={100084}, 
            country={China}}
\ead{aviles-rivero@tsinghua.edu.cn}

\fntext[fn1]{The first two authors contribute equally. Email: cwc56@cam.ac.uk ( Chun-Wun Cheng)}
\cortext[cor1]{Corresponding author. }










\begin{abstract}
Partial differential equations (PDEs) are widely used to model complex physical systems, but solving them efficiently remains a significant challenge. Recently, Transformers have emerged as the preferred architecture for PDEs due to their ability to capture intricate dependencies. However, they struggle with representing continuous dynamics and long-range interactions. To overcome these limitations, we introduce the Mamba Neural Operator (MNO), a novel framework that enhances neural operator-based techniques for solving PDEs. MNO establishes a formal theoretical connection between structured state-space models (SSMs) and neural operators, offering a unified structure that can adapt to diverse architectures, including Transformer-based models. By leveraging the structured design of SSMs, MNO captures long-range dependencies and continuous dynamics more effectively than traditional Transformers. Through extensive analysis, we show that MNO significantly boosts the expressive power and accuracy of neural operators, making it not just a complement but a superior framework for PDE-related tasks, bridging the gap between efficient representation and accurate solution approximation. Our code is available on \hyperlink{https://github.com/Math-ML-X/Mamba-Neural-Operator}{https://github.com/Math-ML-X/Mamba-Neural-Operator}\nocite{*}
\end{abstract}


\begin{keywords}
 \sep Data-driven scientific computing \sep Machine learning \sep Neural Operator \sep Partial Differential Equations
\end{keywords}

\maketitle


\section{Introduction}
Partial differential equations (PDEs) describe various real-world phenomena, such as heat transfer (Heat Equation), fluid dynamics (Navier-Stokes), and biological systems (Reaction-Diffusion). While analytical solutions are sought, many PDEs—like the Navier-Stokes equations—lack closed-form solutions, making them computationally intensive to solve. Numerical methods, such as finite element, finite difference \cite{mehra2010comparison}, and spectral methods, discretise these equations but involve trade-offs between computational cost and accuracy. Coarser grids reduce computational load but sacrifice precision, while finer grids increase both accuracy and computational expense.
Recent advancements in deep learning have changed techniques for solving PDEs. Physics-Informed Neural Networks (PINNs) \cite{raissi2019physics,mattey2021physics} integrate governing equations and boundary conditions into the loss function, but often struggle with generalisation and require retraining for changes in coefficients. Neural operators\cite{bhattacharya2021model,kovachki2023neural}, on the other hand, learn mappings between function spaces, offering a mesh-free, data-driven approach that generalises better across different PDE instances.

Operator learning has gained traction with models like DeepONet\cite{lu2019deeponet} and the Fourier Neural Operator (FNO)\cite{li2020fourier}, which achieved state-of-the-art performance. These models learn input-output mappings to approximate complex operators, similar to sequence-to-sequence problems.
Transformers\cite{vaswani2017attention} have become a go-to architecture for PDEs\cite{cao2021choose,li2022transformer,bryutkin2024hamlet} due to their ability to capture long-range dependencies. This is because the transformers address this challenge by providing a global receptive field: any two spatial tokens can interact in a single layer, making long-range effects as accessible as local ones. The attention mechanism further allows \emph{dynamic weighting}, where the model adaptively selects the most relevant interactions regardless of distance. However, their quadratic complexity limits efficiency for tasks such as long-time integration. To overcome this, efficient variants like Galerkin attention\cite{cao2021choose} reduce computational cost to linear scaling.  
While these models improve efficiency, they trade off model capacity by approximating the self-attention mechanism, potentially reducing accuracy for tasks that need precise attention. Moreover, Transformers face challenges with PDEs due to limited context windows, inefficiency with continuous data, and high memory usage, making them less effective for capturing dependencies over continuous domains and high-resolution grids.

While Transformers are popular for PDE modelling, they have limitations in handling continuous data and high-resolution grids. An emerging alternative is State-Space Models (SSMs)\cite{gu2021combining,gu2022parameterization}, which offer better scalability, reduced memory usage, and improved handling of long-range dependencies in continuous domains compared to Transformers. In particular, Mamba\cite{gu2023mamba} is a novel way designed to effectively capture long-range dependencies, handle continuous data efficiently, and reduce memory consumption in sequence-to-sequence problems. Although Transformers dominate applications like foundational models and computer vision, \textit{the use of SSMs—especially Mamba—for neural operators in PDEs remains underexplored, and their theoretical connections and potential advantages are yet to be fully understood}. 

\textbf{Contributions.} We introduce the concept of Mamba Neural Operator (MNO), which provides a novel perspective applicable to Transformer-based techniques for PDEs. Unlike closely related works, we offer a formal theoretical connection between Mamba and Neural Operators, demonstrating its advantages for PDEs. MNO addresses key challenges in PDE modelling by leveraging its structured state-space design to capture long-range dependencies and continuous dynamics more effectively than Transformers. Our particular contributions are as follows.

We introduce the concept of the Mamba Neural Operator (MNO), where we underline:
\begin{itemize} [noitemsep,nolistsep]
\item Mamba Neural Operator expands the SSM framework into a unified neural operator approach, making it adaptable to diverse architectures, including any Transformer-based model.
\item Unlike existing related works, we provide a theoretical understanding that shows how neural operator layers share a comparable structural framework with time-varying SSMs, offering a new perspective on their underlying principles.
\end{itemize}

We evaluate MNO on various architectures and PDEs, showing through systematic analysis that Mamba enhances the expressive power and accuracy of neural operators. This indicates that Mamba is not just a complement to Transformers, but a superior framework for PDE-related tasks, bridging the gap between efficient representation and accurate solutions.

\section{Related Work}
\fcircle[fill=deblue]{2pt}\textbf{ Data-Driven PDEs.}
Recent advances in fluid dynamics and solving PDEs have led to architectures modelling continuous-time solutions and multiparticle dynamics\cite{kochkov2021machine, lusch2018deep}. Physics-informed models now offer solutions in unsupervised and semi-supervised settings\cite{raissi2019physics,li2020fourier}. These models typically encode spatial data and evolve over time, utilising methods like convolutional layers\cite{ronneberger2015u,wiewel2019latent} symbolic neural networks\cite{udrescu2020ai}, and residual networks\cite{he2016deep}. Finite element methods (FEM), including Galerkin and Ritz, are also integrated into learning frameworks\cite{chen2021continuous}.

\fcircle[fill=deblue]{2pt}\textbf{ Neural Operators.} 
Neural operators, such as the Graph Neural Operator\cite{li2020neural} and Fourier Neural Operator\cite{li2020fourier}, excel at learning mappings in infinite-dimensional spaces, particularly by leveraging techniques like graph structures or transformations in Fourier space. The Fourier Neural Operator (FNO) and its variations, including the incremental, factorised, adaptive FNO, and FNO+\cite{zhao2022incremental, tran2021factorized, guibas2021adaptive} have shown exceptional performance in both speed and accuracy. Their key advantage lies in their ability to maintain discretisation invariance, which sets them apart in many applications. DeepONet\cite{lu2019deeponet} pioneered the nonlinear operator approximation using separate networks for inputs and query points, while extensions like MIONet handle multiple inputs\cite{jin2022mionet}. Challenges like irregular grids are being addressed through grid mapping and subdomain partitioning\cite{li2022irregular,wen2022u} though scalability for diverse inputs remains a key focus.

\fcircle[fill=deblue]{2pt}\textbf{ Transfomers for PDEs.}
The Transformer model\cite{vaswani2017attention} stands out due to its distinctive features, primarily its use of attention mechanisms to model the relationships among input elements. Initially, it was developed for NLP, and attention mechanisms have been adapted to PDEs, providing flexible and efficient mappings between function spaces. Recent theoretical work \cite{yun2019transformers} also establishes that transformers are universal approximators of sequence-to-sequence functions, capable of representing arbitrarily complex dependencies.
Galerkin attention\cite{cao2021choose} introduced linear complexity to reduce computational costs, inspiring further developments like GNOT\cite{hao2023gnot} and OFormer\cite{li2022transformer}, which achieve state-of-the-art results. Additionally, graph-based Transformers have also been explored to capture complex interactions in irregular domains\cite{bryutkin2024hamlet}.

\fcircle[fill=deblue]{2pt}\textbf{ State-Space Models for PDEs \& Comparison to Ours.}
Initial studies on SSMs for PDEs, like MemNO\cite{buitrago2024benefits}, explored combining FNO with S4 but were restricted to low-resolution or noisy inputs. In contrast, we introduce the Mamba Neural Operator, which generalises the SSM framework to neural operators, making it compatible with any architecture, including Transformers. Our approach extends the theoretical foundations for broad applicability to any PDE family, highlighting Mamba’s effectiveness in diverse scenarios.
At the time of our submission, the work of that\cite{hu2024state} proposed integrating state-space models into neural operators for dynamical systems. While related, our work differs significantly, their approach focuses on dynamical systems and tests only on ordinary differential equations (ODEs), whereas we target parametric partial differential equations (PDEs). Additionally, we provide a theoretical understanding showing that neural operator layers share a comparable structural framework with time-varying SSMs, demonstrating alignment between hidden space updates and the iterative process in neural operators.

\section{Mamba Neural Operator}
This section details the theoretical underpinning and practicalities of the Mamba Neural Operator. We outline its design, key components, and operational mechanisms, explaining how it efficiently models partial differential equations by leveraging structured state-space models (SSMs).

\subsection{Problem Statement}
We consider parametric partial differential equations (PDEs) defined on a domain $\Omega \subset \mathbb{R}^n$, parameterised by $\theta \in S \subset \mathbb{R}^p$, where $\theta$ is sampled from a distribution $w$. The general form of the PDE is:
\begin{equation}
\begin{aligned}
P : \mathcal{P} \times  \Omega \times  \mathcal{W} \times \mathbb{R}^m \times \ldots \times \mathbb{R}^m \to \mathbb{R}^\ell, \quad \Omega \subset \mathbb{R}^n,  W \subset \mathbb{R}^m, \\
P(\theta, x, u,  \partial_{x_1} u, \ldots, \partial_{x_n} u, \ldots, \partial^{\beta_1}_{x_1} \cdots \partial^{\beta_n}_{x_n} u) = 0,
\end{aligned}
\end{equation}
where the unknown function $u : \Omega \rightarrow V$ solves $P$. The multi-index $\beta = (\beta_1, \ldots, \beta_n)$, with $|\beta| = \sum_{i=1}^n \beta_i$, determines the differentiation orders. If time is involved, $\Omega$ reduces to $\mathcal{T} \subset \mathbb{R}_{\geq 0}$ and $\Omega \subset \mathbb{R}^{n-1}$. To ensure well-posedness, initial and boundary conditions must hold:
\begin{equation}
\begin{aligned}
u(x, T_0) = u_0(x), \; x \in \Omega_\theta, \quad \quad
u(x, t) = u_b(x),  \; x \in \partial \Omega_\theta, \; t \in \mathcal{T},
\end{aligned}
\end{equation}
for $x \in \Omega_\theta$ and $t \in \mathcal{T}$, where $u_0$ and $u_b$ are the initial and boundary conditions, respectively. Assume $\Omega, \mathcal{P}, V$ are Banach spaces, and there exists an analytic solution operator: $O: \mathcal{P} \times \Omega \times \mathbb{R}^m \times \ldots \times \mathbb{R}^m \times \mathbb{R}^\ell \times \mathbb{R}^\ell \rightarrow V.$
Our aim is to design a neural network $\tilde{S}\mu : (\theta, u_0, u_b) \mapsto u$ that approximates this operator, with $\mu$ as the network’s parameters. Given a dataset ${(\theta^{(n)}, u^{(n)})}{n=1}^N$, where $\theta^{(n)}$ and $u^{(n)}$ correspond to the system's discretised parameters, we simplify the notation as $\theta^{(n)} = \theta(x^{(n)})$ and $u^{(n)} = u(x^{(n)})$.

\subsection{Preliminaries: Transformer and Mamba}
Transformers have emerged as the leading architecture for many state-of-the-art techniques in solving PDEs. Mamba, on the other hand, serves as a promising alternative to Transformers. In this section, we provide an overview of the background of Transformers and State Space Sequence Models (SSMs).

\textbf{Transformer.} In each Transformer layer, an attention mechanism enables interaction between inputs at varying positions, followed by a position-wise fully connected network applied independently to each position. Specifically, the attention mechanism involves projecting an intermediate representation into three components—query $Q \in \mathbb{R}^{N\times d_k}$, key $K \in \mathbb{R}^{N\times d_k}$, and value $V \in \mathbb{R}^{N\times d_v}$—using three separate position-wise linear layers. These representations are then used to calculate the output as:
\begin{equation}
\begin{aligned}
  \Attention(Q,K,V) = \softmax\bigg(\frac{QK^T}{\sqrt{d_k}}\bigg)V,
\end{aligned}
\end{equation}
of which the memory complexity is $O(n^2)$. To reduce the computational inefficiency, Galerkin-type attention was proposed by\cite{cao2021choose} to remove Softmax attention with linear complexity. It defines as follows:
\begin{equation}
\begin{aligned}
  \Attention_{g}(Q,K,V) = \frac{Q(K_{g}^{T} V_{g})}{d},
\end{aligned}
\end{equation}

in which $K_{g}^{T}$ and $V_{g}$ denote layer normalisation of K and V, as described in\cite{lei2016layer}. The Galerkin-type attention mechanism involves two matrix product operations, resulting in a computational complexity of \( O(nd^2) \). This reduces the sequence length dependency to only $O(n)$.

\textbf{State Space Sequence Models (SSMs).} Structured State Space (S4) models introduce a new approach in deep learning sequence modelling, incorporating elements from Recurrent Neural Networks (RNNs), Convolutional Neural Networks (CNNs), and classical state space models. These models are inspired by control theory, where the process involves mapping an input sequence $u(t) \in \mathbb{R}^L$ to an output sequence $y(t) \in \mathbb{R}^L$ through a hidden latent state $h(t) \in \mathbb{R}^N$. The core mechanism of State Space Models (SSMs) is formulated using linear first-order ordinary differential equations, enabling efficient handling of temporal data, which reads:
\begin{equation}\label{mamba}
\begin{aligned}
h'(t) = Ah(t) + Bu(t), \quad
y(t) = Ch(t) + Du(t),
\end{aligned}
\end{equation}

where $A \in \mathbb{C}^{N \times N}$ and $B, C \in \mathbb{C}^N$ and  \( D \in \mathbb{C}^{N} \). Mamba, a more advanced variant of SSMs, refines this formulation by incorporating efficient state space parameterisation and selection mechanisms. Unlike earlier models such as S4, which uses bilinear method, Mamba adopts zero-order holds, allowing it to handle larger hidden states and longer sequences more effectively. This makes Mamba particularly well-suited for complex sequence modeling tasks, such as natural language processing and time-series analysis.

\subsection{State Space Models Discretisation for PDEs}
State Space Models (SSMs) have emerged as a strong alternative to Transformers in deep learning. While Transformers dominate in areas like foundational models and computer vision, the application of SSMs, particularly the Mamba architecture, to neural operators for PDEs is still underexplored.

We start by demonstrating that the discretisation of S6 (Mamba) is equivalent to the well-known Euler method when the Taylor series expansion is applied. Mamba utilises zero-order holds, resulting in the following discretisation, which reads: 
\begin{equation}\label{Zero}
\begin{aligned}
    A = \exp(\Delta A), \quad B = (\Delta A)^{-1} \left( \exp(\Delta A) - I \right) \cdot \Delta B.
\end{aligned}
\end{equation}

To discretise the continuous-time SSM, we apply the zero-order hold (ZOH) method, a standard control-theoretic discretisation technique that assumes the input signal $u(t)$ remains constant over each sampling interval. Given a continuous-time linear state-space system:
\begin{equation}
\dot{h}(t) = A h(t) + B u(t), \quad y(t) = C h(t),
\end{equation}
the ZOH discretisation produces an equivalent discrete-time system of the form:
\begin{equation}
h_{k+1} = \tilde{A} h_k + \tilde{B} u_k,
\end{equation}
where:
\[
\tilde{A} = e^{A \Delta}, \quad \tilde{B} = \left( \int_0^{\Delta} e^{A \tau} d\tau \right) B = A^{-1}(e^{A \Delta} - I) B.
\]
This is the formulation used in Equation (6) of our paper. ZOH is particularly relevant to our work because it preserves the system dynamics while enabling a principled transition from continuous-time to discrete-time models—crucial when applying SSMs to neural operators operating on discretized PDE data.
Under ZOH, the discrete-time equivalent of the system is given by:
\begin{equation}
\tilde{A} = e^{A \Delta}, \quad \tilde{B} = A^{-1}(e^{A \Delta} - I) B,
\end{equation}
where \( \Delta \) is the step size. This formulation preserves the dynamics of the continuous system and forms the basis for our discrete SSM implementation in~\eqref{Zero}.

\textbf{Discretisation of SSM.} To incorporate the SSM into deep learning frameworks, we need to transform the continuous-time SSM into a discrete formulation. This is done by expressing the continuous-time system as an ordinary differential equation (ODE) and then solving it numerically. As discussed in\cite{liu2024vmamba}, the discrete SSM reads:
%
\begin{equation}
\label{ssm}
\begin{split}
h_{a + 1} &= e^{A\Delta_a}(h_a + B_{a}u_{a}e^{- A\Delta_a\Delta_a})\\
 &= e^{A\Delta_a} h_a + B_{a}\Delta_a u_a  
= \Bar{A}_a h_a + \Bar{B}_{a}u_a,
\end{split}
\end{equation}
where $\Delta$ is the time step size, and $\Bar{A}_a = e^{A\Delta_a}$ and $\Bar{B} = B_{a}\Delta_a$ are the discretised system matrices.
%
In the S6 model, we define  $\Tilde{A}=e^{\Delta A}$ and $\Tilde{B}= (\Delta A)^{-1}(e^{\Delta A} - I)\cdot \Delta B$. By applying sampling in $\Tilde{A}$, we have $\Bar{A} = \Tilde{A}$.
For $\Tilde{B}$, applying the sampling process yields to:
\begin{equation}
\begin{split}
\Tilde{B} &= (\Delta A)^{-1}(e^{\Delta A} - I)\cdot \Delta B \\
          &= (\Delta A)^{-1}(I + \Delta A + O(\Delta^2) -I)\cdot \Delta B \\
          &= (\Delta A)^{-1}(\Delta A + O(\Delta^2))\cdot \Delta B
          =  \Delta B (Drop~ O(\Delta^2)) 
          =  \Bar{B}
\end{split}
\end{equation}
Thus, we have shown that our discretisation method is equivalent to the zero-order hold method, where $\Bar{A}=\Tilde{A}$ and $\Bar{B}=\Tilde{B}$.

\begin{proposition}~\label{thm1}
The zero-order hold discretisation method, as in~\eqref{Zero}, is equivalent to the Euler method in SSM when the Taylor series expansion of the exponential function is truncated to its first-order term. 
\end{proposition} 
\begin{proof}
In SSM, if we define the matrices as
$\Hat{A} = I + \Delta A$ and $\Hat{B} = \Delta B$. Then the discretised form of the state update can be written as: 
\begin{equation}
\begin{split}
h(t+ \Delta )  &= \Hat{A}h(t) + \Hat{B}u(t) 
                = (I + \Delta A)h(t) + \Delta Bu(t) \\
                &= h(t) + \Delta (A h(t) + Bu(t)) 
                = h(t) + \Delta h'(t).
\end{split}
\end{equation}
which implies it is a first order Euler method.
It is straightforward to show that
$\Tilde{A} = \Hat{A}$ since $\Tilde{A} = e^{A \Delta } = I + A \Delta + O( \Delta ^2) = I + A \Delta = I + A \Delta =  \Hat{A}  $. Similarly, we observe that  $\Tilde{B} = \Bar{B} = \Hat{B}$.
Therefore, the discretisation used in the SSM method can be replaced with the zero-order hold method by substituting $\Hat{A} = \Tilde{A}$ and $\Hat{B} = \Tilde{B}$, we get:
\begin{equation}
\begin{split}
                h(t+ \Delta )  
                &= \Hat{A}h(t) + \Hat{B}u(t) 
                = \Tilde{A}h(t) + \Tilde{B}u(t) \\
                &= (e^{\Delta A})h(t) + ((\Delta A)^{-1}(e^{\Delta A} - I)\cdot \Delta B)u(t) \\
                &= (I + \Delta A + O(\Delta ^2) h(t) + (\Delta B) u(t) \\
                &= (I + \Delta A ) h(t) + (\Delta B) u(t) \\
                &= h(t) + \Delta Ah(t) + \Delta Bu(t) \\
                &= h(t) + \Delta (Ah(t) + Bu(t)) 
                = h(t) + \Delta h'(t).
\end{split}
\end{equation}
This shows that the zero-order hold discretisation method is equivalent to the Euler method, as both yield the same discrete update formula.

\end{proof}

{\textbf{Why is Proposition~\ref{thm1} important for PDEs?}} Proposition~\ref{thm1}, which establishes the equivalence between the Zero-Order Hold (ZOH) method and the Euler method, is crucial for understanding Mamba’s performance in solving partial differential equations (PDEs). This equivalence demonstrates that ZOH can be viewed as a more generalised and accurate variant of the Euler method. Throughout this work \(\Delta >0\) denotes the time step for the mamba used to advance the hidden state of a continuous-time state–space model (SSM).  
Applying the forward-Euler scheme to the linear ODE
$h'(t)=Ah(t)+Bu(t)$ gives the following update
\begin{equation}
h(t+ \Delta)=h(t)+\Delta \bigl(Ah(t)+Bu(t)\bigr)
\end{equation},
so Euler’s method uses a time increment \(\Delta\).
Proposition~\ref{thm1} shows that the Zero-Order Hold (ZOH) discretisation,
\[
\tilde A=\mathrm e^{A\Delta},\qquad 
\tilde B=\bigl(A^{-1}\!(\mathrm e^{A\Delta t}-I)\bigr)B,
\]
reduces to the Euler matrix pair \((I+\Delta A,\;\Delta B)\) when the matrix exponential is truncated at first order.  
Hence, ZOH can be viewed as a higher-order method of forward Euler: it retains the leading terms of the exact flow map \(\mathrm e^{A\Delta}\) and therefore achieves a local error \(O(\Delta^{2})\) instead of \(O(\Delta)\). In the supplement, we provided an example to explain the difference between the Euler method and the ZOH method.

In this work, we provide the first analysis of the Zero-Order Hold method, showing that it could be extended as a higher-order method via Taylor series expansion and reduces to the first-order Euler scheme when truncated at first order. To the best of our knowledge, this perspective is the first work that provides a connection between classical numerical methods and Zero Order Hold method. Although the Forward Euler method is the simplest first-order scheme, it is only conditionally stable and limited in accuracy. Higher-order methods can achieve faster convergence and improved accuracy, though often with different stability characteristics. By reforming ZOH as a Taylor series method, we open the possibility to extend it to higher-order methods. This is particularly important because the choice of time step and ODE solver directly governs how the state-space model propagates forward in time, and thus has a critical impact on the accuracy of PDE solutions.

\subsection{Network Architecture}
As depicted in Figure \ref{fig:overview}, the data processing pipeline in our Mamba Neural Operator (MNO) is composed of three key stages: Bi-Directional Scan Expand, S6/Cross S6 Block, and Bi-Directional Scan Merge. As shown in Figure \ref{fig:compar}, our method uses Mamba and bidirectional scan to achieve linearly complex and global receptive fields, effectively combining the strengths of CNNs and Vision Transformers (ViTs). Unlike CNNs, which capture only local dependencies, and ViTs, which model global context but come with quadratic complexity, Mamba (2D) introduces a bi-scan mechanism that enables efficient global information propagation. This design provides linear complexity to CNNs while retaining the expressive global modeling ability of ViTs. When solving PDEs over a fixed grid, the input data can be structured as grid-based data, similar to an image. In the first stage, Bi-Directional Scan Expand, the MNO unfolds the input data into sequences by traversing the grid along two distinct paths. These sequences, representing input patches, are processed independently in the next step. The second stage, S6/Cross S6 Block, involves processing each patch sequence using either an S6 or Cross S6 block, depending on the model variation being employed. For instance, in the enhanced version of Mamba, the GNOT model utilises a Cross S6 block followed by an S6 block for further refinement. Finally, in the Bi-Directional Scan Merge stage, the processed sequences are reshaped and merged back together to generate the output map, completing the data forwarding process. This structured approach allows the MNO to efficiently handle grid-based input data, enabling scalable solutions for PDEs.

The S6 Block has the same definition for mamba, while the Cross S6 block is the new block. We provide the definition here.
\begin{definition} \textbf{(Cross S6 Block):}   Let \( x \) and \( x' \) be two independent input vectors. Each input is processed through two independent linear transformation, resulting in corresponding parameter sets \((B, C, \Delta)\) for \( x \) and \((B', C', \Delta')\) for \( x' \). Specifically, these transformations are defined as:

\begin{equation}
\begin{aligned}
B, C, \Delta &= \text{Linear}_x(x), \\
B', C', \Delta' &= \text{Linear}_{x'}(x'),
\end{aligned}
\end{equation}
where \( \text{Linear}_x \) and \( \text{Linear}_{x'} \) are the respective linear layers applied to \( x \) and \( x' \).

Next, the parameters \((\tilde{B}, \tilde{C}, \tilde{\Delta})\) are computed by combining the updated values from both inputs according to the following equations:

\begin{equation}
\begin{aligned}
\tilde{B} &= B + q B', \\
\tilde{C} &= C + q C', \\
\tilde{\Delta} &= \Delta + q \Delta',
\end{aligned}
\end{equation}
where \( q \) is a scalar ratio controlling the contribution of the second input \( x' \) to the combined output. Once we have these updated parameters, we apply the State Space Model (SSM) to compute the final output y.
\end{definition}

\subsection{Mamba for Neural Operators}

Neural Operators\cite{li2020neural} aim to learn mappings between function spaces, providing a framework for solving partial differential equations (PDEs) and other problems involving continuous functions. It updates the value by an iterative method: $i_0 \rightarrow i_1 \rightarrow \ldots \rightarrow i_T,$ where each \(i_j\) (for \(j = 0, 1, ..., T-1\)) maps to \(\mathbb{R}^{d_v}\). Let the input be $a(x)$ and the output be $u(x)$ . The input \(a\), drawn from set \(A\), is initially lifted to a higher-dimensional representation: $v_0(x) = P(a(x))$ where $P$ is a local transformation, typically parameterised by a fully-connected neural network. We then apply iterations to update $i_t \rightarrow i_{t+1}$ as defined in Definition 1.  The final output: $u(x) = Q(v_T(x))$ is the result of projecting $v_T$ via the transformation: $Q : \mathbb{R}^{d_v} \to \mathbb{R}^{d_u}.$ Each update from \(i_t\) to \(i_{t+1}\) involves the integration of a non-local integral operator $K$ and a local nonlinear activation function $\sigma$. One of the main results of this work is establishing the equivalence between neural operators and the Mamba framework. Therefore, we first introduce fundamental definitions stated in \cite{li2020neural} that are essential for demonstrating this relationship.

\begin{figure*}[t!]
    \centering
    \includegraphics[width=\linewidth]{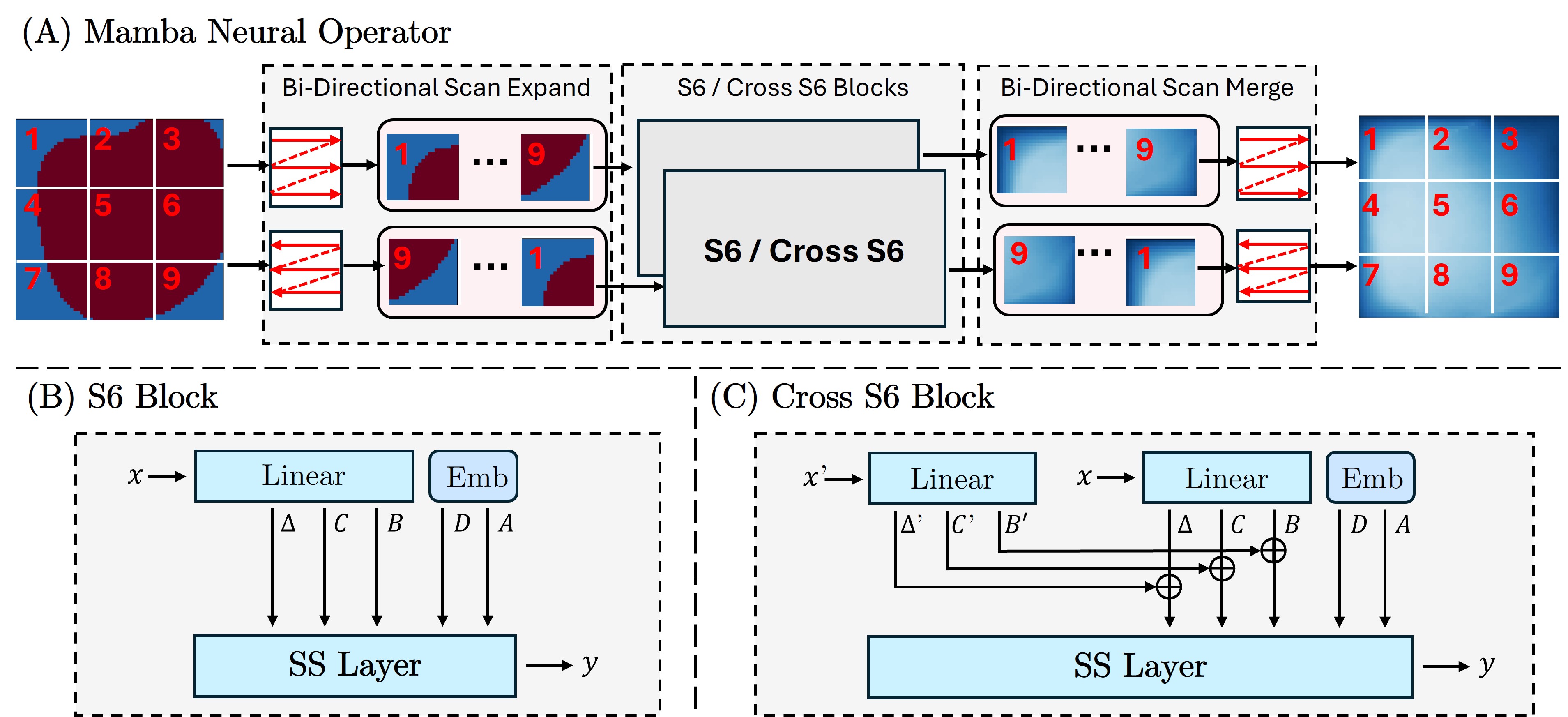}
    \caption{(A) Illustration of Mamba Neural Operator. Input image patches are processed by following two distinct scanning paths (referred to as Bidirectional -Scan). Each sequence generated from these paths is passed through separate S6 blocks/ Cross S6 Blocks for independent processing. Afterwards, the outputs from the S6 blocks / Cross S6 Blocks are combined to form a feature map, resulting in the final output (Bidirectional-Merge). (B) and (C) are the detailed blocks of the S6 Block and Cross S6 Block, respectively.  }
    \label{fig:overview}
\end{figure*}
\begin{definition}
 \label{def1}
  \textbf{(Iterative updates)}: The update from $i_t \rightarrow i_{t+1}$ is defined as follows:
\begin{equation}
\begin{aligned}
i_{t+1}(x) := \sigma\left( W i_t(x) + K_{\phi}(a) i_t(x) \right), \quad \forall x \in D,
\end{aligned}
\end{equation}
\end{definition}
\begin{definition}
 \label{def1}
  \textbf{(Iterative updates)}: The update from $i_t \rightarrow i_{t+1}$ is defined as follows:
\begin{equation}
\begin{aligned}
i_{t+1}(x) := \sigma\left( W i_t(x) + K_{\phi}(a) i_t(x) \right), \quad \forall x \in D,
\end{aligned}
\end{equation}

where $K : A \times \Theta_K \rightarrow L(U(D; \mathbb{R}^{d_v}), U(D; \mathbb{R}^{d_v}))$ represents a mapping to bounded linear operators on $U(D; \mathbb{R}^{d_v})$, parameterised by $\phi \in \Theta_K$. The function $W : \mathbb{R}^{d_v} \rightarrow \mathbb{R}^{d_v}$ is a linear transformation, and $\sigma : \mathbb{R} \rightarrow \mathbb{R}$ is a nonlinear activation function applied component-wise.
\end{definition}

\begin{figure*}[h!]
    \centering
    \includegraphics[width=\linewidth]{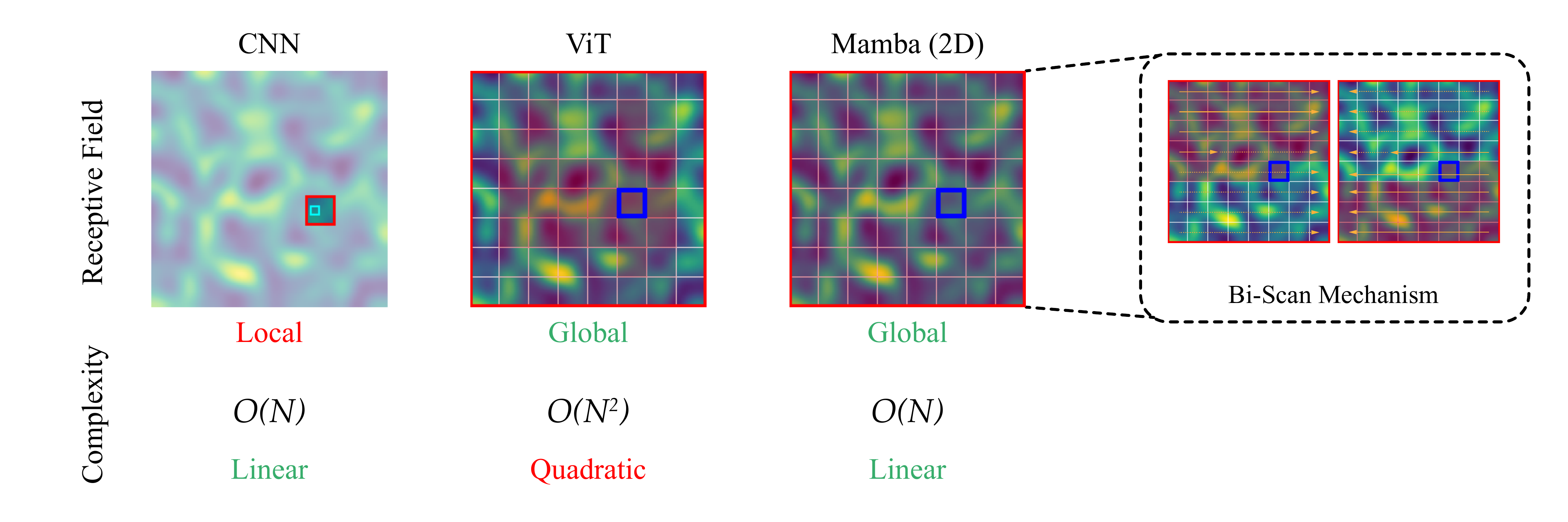}
    \caption{Comparison of receptive fields and computational complexity across architectures. CNNs capture only local context with linear complexity, while ViTs achieve global modeling but with quadratic complexity. Our proposed Mamba (2D) attains global receptive fields with linear complexity by leveraging a bi-scan mechanism, combining efficiency with expressive global context.}
    \label{fig:compar}
\end{figure*}

\begin{definition}
\label{def2}
\textbf{(Kernel integral operator \( K \))}: Define the kernel integral operator mapping in~\ref{def1} by
\begin{equation}
\begin{aligned}
K_{\phi}(a) i_t(x) := \int_D \kappa_{\phi}(x, y, a(x), a(y)) i_t(y) \, dy, \quad \forall x,
\end{aligned}
\end{equation}
where \(\kappa_{\phi}: \mathbb{R}^{2(d + d_a)} \rightarrow \mathbb{R}^{d_v \times d_v}\) is a neural network parameterised by \(\phi \in \Theta_K\).
\end{definition}

As mentioned in the previous section, we can be discrete SSM into the form of \eqref{ssm}. This representation can be  rewrite as\cite{liu2024vmamba}: $ \mathbf{h}_b = \mathbf{w}_T \odot \mathbf{h}_a + \sum_{i=1}^T \frac{\mathbf{w}_T}{\mathbf{w}_i} \odot (\mathbf{K}_i^\top \mathbf{V}_i).$
We define $\mathbf{V} = [\mathbf{V}_1; \ldots; \mathbf{V}_T] \in \mathbb{R}^{T \times D_v}$, where $\mathbf{V}_i = u_{a+i-1} \Delta_{a+i-1} \in \mathbb{R}^{1 \times D_v}$, $\mathbf{K} = [\mathbf{K}_1; \ldots; \mathbf{K}_T] \in \mathbb{R}^{T \times D_k}$, where $\mathbf{K}_i = \mathbf{B}_{a+i-1} \in \mathbb{R}^{1 \times D_k}$ , and $\mathbf{Q} = [\mathbf{Q}_1; \ldots; \mathbf{Q}_T] \in \mathbb{R}^{T \times D_k}$, where $\mathbf{Q}_i = \mathbf{C}_{a+i-1} \in \mathbb{R}^{1 \times D_k}$. We further define $\mathbf{w} = [\mathbf{w}_1; \ldots; \mathbf{w}_T] \in \mathbb{R}^{T \times D_k \times D_v}$, where $\mathbf{w}_i = \prod_{j=1}^{i} e^{\mathbf{A} \Delta_{a-1+j}} \in \mathbb{R}^{D_k \times D_v}$, and $\mathbf{H} = [\mathbf{h}_a; \ldots; \mathbf{h}_b] \in \mathbb{R}^{T \times D_k \times D_v}$, where $\mathbf{h}_i \in \mathbb{R}^{D_k \times D_v}$. Finally, we set $\mathbf{Y} = [\mathbf{y}_a; \ldots; \mathbf{y}_b] \in \mathbb{R}^{T \times D_v}$, where $\mathbf{y}_i \in \mathbb{R}^{D_v}$ 

This formulation indicates that Gated Linear Attention\cite{yang2023gated} is actually a specific variant of Mamba.
\textit{We next present our main result is how neural operator layers share a comparable structural framework with time-varying SSMs, which, to the best of our knowledge, is established here for the first time.}
\begin{proposition}\label{Prop:2}
The hidden space in time-varying state-space models demonstrates a structural similarity to neural operator layers.
\end{proposition} 
\begin{proof}
\textbf We first rewrite the time-varying SSMs (\ref{ssm}) as:
\begin{equation}
\begin{aligned}
    \mathbf{h}_b = \mathbf{w}_T \odot \mathbf{h}_a + \sum_{i=1}^T \frac{\mathbf{w}_T}{\mathbf{w}_i}
 \odot (\mathbf{K}_i^\top \mathbf{V}_i),
    \label{equ:Re}
\end{aligned}
\end{equation}
where $\mathbf{w}_T, \mathbf{w}_i, \mathbf{K}_i, \mathbf{V}_i$ are as defined previously.

To demonstrate that the hidden space update in our Mamba Operator has a similar structural framework to neural operator layers, we assume the shapes of $\mathbf{w}$ and $\mathbf{h}$ are $(T, D_k)$, represented as vectors. Our goal is to show that the iterative process in~\eqref{ssm} aligns with that of Definition~\ref{def1}. Consider the first part of Definition~\ref{def1}, represented by $W_i(x)$. We set $W = \mathbf{W}_T$ and $i_t(x) = \mathbf{h}_a$, where $\mathbf{h}_a$ is the hidden state from the previous iteration. We then verify that $\mathbf{W}_T$ satisfies the properties of a linear transformation, ensuring consistency with the neural operator framework.

We can proved this as follows: without loss of generality,  let us assume that $W_T = W_1 =  e^{\mathbf{A} \Delta_{a}}$. Next, we apply this transformation to a vector $x$ and check the conditions for linearity: $T(x+y) = T(x) + T(y)$ and $T(ax)= aT(x)$. By applying the Taylor expansion to $e^{\mathbf{A} \Delta_{a}}$ , then we get $I + A \Delta_a + O(\Delta_a ^2) $. To show that $T(x+y) = T(x) + T(y)$, it suffices to demonstrate: $e^{A \Delta_a}(x_1 + x_2) = e^{A \Delta_a}(x_1) + e^{A \Delta_a}(x_2)$, we have:
$
e^{\Delta A_{\alpha}(x_1 + x_2)} = 
\left(I + \Delta A_{\alpha} + O(\Delta^2_{\alpha})\right)(x_1 + x_2) 
= I(x_1 + x_2) + \Delta A_{\alpha}(x_1 + x_2) + O(\Delta^2_{\alpha})(x_1 + x_2).$

This shows $T(x+y) = T(x) + T(y)$. For the second condition, we want to show $T(\alpha x)=\alpha T(x)$, which is equivalent to demonstrating that $e^{A \Delta_a}(\alpha x) = \alpha e^{A \Delta_a}x $, we get:
\begin{equation}
\begin{aligned}
e^{A \Delta_a}(\alpha x)  &= (I + A \Delta_a + O(\Delta_a ^2))(\alpha x)\\
                &= I\alpha x + A \Delta_a \alpha x + O(\Delta_a ^2)\alpha x  \\
                &= \alpha (I x + A \Delta_a  x + O(\Delta_a ^2) x)  
                = \alpha (e^{A \Delta_a} x) 
\end{aligned}
\end{equation}
Thus, we have shown that $e^{A \Delta_a}$ satisfies the two conditions, and hence it is a linear transformation. This shows the update in hidden space is the same as neural operator.

Secondly, we need to check the second part of Definition~\ref{def1}, which involves showing that:
\begin{equation}
\begin{aligned}
K_{\phi}(a) i_t(x) :=  \sum_{i=1}^T \frac{\mathbf{w}_T}{\mathbf{w}_i} \odot (\mathbf{K}_i^\top \mathbf{V}_i)
\end{aligned}
\end{equation} has a similar structure.

According to Definition~\ref{def2}, it suffices to demonstrate that:
$\int_D \kappa_{\phi}(x, y, a(x), a(y)) i_t(y) \, dy =  \sum_{i=1}^T \frac{\mathbf{w}_T}{\mathbf{w}_i} \odot (\mathbf{K}_i^\top \mathbf{V}_i)$. We assume the kernel $\kappa_{\phi}$ can be decomposed into a finite sum of separable basis functions: $\kappa_{\phi}(x, y, a(x), a(y)) = \sum_{i=1}^{T} \omega_i \varphi_i(x)\psi_i(y)
$ such that $\omega_i$ is learnable weights for each basis function.
and Basis functions capturing interactions between $x$ and $y$. Then we substitute it  into the integral such that : $\int_{D}  \sum_{i=1}^{T} \omega_i \varphi_i(x)\psi_i(y) v_t(y) \, dy = \sum_{j=1}^{T} \omega_i \int_{D} \varphi_i(x)\psi_i(y) v_t(y) \, dy$. We further discretise the domain 
$D$ into $T$ points  $\left\{ \mathbf{y}_i \right\}_{i=1}^{T}$ with corresponding weights $\Delta y$. The integral becomes $K_{\phi}(a) i_t(\mathbf{x}) \approx \sum_{i=1}^{T} \omega_i \sum_{j=1}^{T}\varphi_i(x)\psi_i(y_j) i_t(\mathbf{y}_j) \Delta y$. We represent 
$\varphi_i(\mathbf{x})$ as vector $\mathbf{K}_i$ and the input $i_t(\mathbf{y}_i)$
as vector $\mathbf{V}_i$: $\mathbf{K}_i = \left[ \varphi_i(\mathbf{x}), \varphi_i(\mathbf{x}), \dots, \varphi_i(\mathbf{x}) \right]^\top \in \mathbb{R}^{1 \times D_k}, \quad \mathbf{V}_i = \left[ \psi_i(y_1)i_t(y_1) \Delta y, \dots, \psi_i(y_T)i_t(y_T) \Delta y \right]^\top \in \mathbb{R}^{1 \times D_k}$. If we further factorise $\omega_i$ as $\frac{w_T}{w_i}$, where $w_T$  is a hyperparameter and $w_i$ represents a  set of parameters to be learned, we obtain the update: 
$K_{\phi}(a) i_t(x) :=  \sum_{i=1}^T \frac{\mathbf{w}_T}{\mathbf{w}_i} \odot (\mathbf{K}_i^\top \mathbf{V}_i)$.  Consequently, the neural operator layer shares a comparable structural framework with time-varying SSMs, demonstrating that the hidden space update in these models aligns with the iterative process in neural operator layers.

FNO  and DeepONet are known for their universal approximation properties and discretisation invariance. These properties enable them to generalise well across various PDE discretisations, but they do not specifically account for the connection between the model and the discretisation of the PDE being solved. In contrast, MNO is built on {state-space models (SSMs)} \cite{doyle2013feedback}, naturally formulated in continuous time, which provide an alignment to discretisation of the PDE. This allows MNO to leverage the {structure of the underlying system} in a way that FNO and DeepONet do not, ensuring {better alignment with the problem-specific discretisation}. 

The state-space formulation enables explicit control over discretisation-induced errors \cite{basar2001new}, providing {theoretical guarantees on convergence} and {stability} that are not inherently available in universal approximator-based methods like FNO and DeepONet. Furthermore, while FNO and DeepONet can learn the solution operator through deep neural networks or Fourier-based methods, these approaches do not explicitly model the {temporal dynamics} or {interactions} within the underlying PDE system in the same way MNO does. By framing the problem in terms of a state-space system, MNO provides a more {robust and interpretable framework}, especially when scaling to higher resolutions or dealing with complex boundary conditions, which are common in real-world PDE applications. 

\end{proof}

\section{Experiments and Discussion} 
In this section, we thoroughly describe the implementation setup and present experimental results to validate Mamba Neural Operators along with Transformers.

\subsection{Dataset Description \& Implementation Protocol}
 \textbf{PDEs Selection.} 
We utilise datasets from PDEBench\cite{takamoto2022pdebench}, a publicly available benchmark for partial differential equations (PDEs). We focus on three PDEs representing both stationary and time-dependent problems: Darcy Flow, Shallow Water 2D (SW2D), Diffusion Reaction 2D (DR2D), and Compressible Navier-Stokes equation 2D (CFD2D). All simulations are performed on a uniform grid. We also included ND2D-c which has a irrengular mesh and is multi-scale. Detailed information about the datasets is defined as follows: 
\paragraph{\fcircle[fill=deblue]{2pt} Darcy Flow.} The two-dimensional Darcy Flow equation defines as follows:
\begin{equation}
\begin{cases}
-\nabla \cdot \left( a(x, y) \nabla u(x, y) \right) = f(x, y), & \text{for } (x, y) \in \Omega, \\
u(x, y) = 0, & \text{for } (x, y) \in \partial \Omega,
\end{cases}
\end{equation}
where \( a(x, y) \) is the diffusion coefficient, \( u(x, y) \) is the solution respectively, and \( \Omega = (0, 1)^2 \) is a square domain. In Darcy Flow, the force term \( f(x, y) \) is set to be a hyperparameter \( \beta \), which influences the scale of the solution \( u(x, y) \).
Experiments were performed on the steady-state solution of the 2D Darcy Flow over a uniform square domain. The goal is to approximate the solution operator \( S \) defined by:
\begin{equation}
\begin{aligned}
S: a \mapsto u, \quad \text{for } (x, y) \in \Omega,
\end{aligned}
\end{equation}
with \( a(x, y) \) and \( u(x, y) \) as previously defined. Similar as PDEBench\cite{takamoto2022pdebench} protocol, we used only $\beta =  1.0$ and we divided the training and testing ratio into 9:1 which contains 9,000 samples for training and 1,000 samples for testing.  

\paragraph{\fcircle[fill=deblue]{2pt} Shallow Water.} We conducted experiments on the two-dimensional Shallow Water equations, which are effective for modeling free-surface flow problems. The equations are formulated as follows:
\begin{equation}
\begin{aligned}
\partial_t h + \partial_x (h u) + \partial_y (h v) &= 0, \\
\partial_t (h u) + \partial_x \left(u^2 h + \tfrac{1}{2} g_r h^2\right) &= -g_r h \, \partial_x b, \\
\partial_t (h v) + \partial_y \left(v^2 h + \tfrac{1}{2} g_r h^2\right) &= -g_r h \, \partial_y b,
\end{aligned}
\end{equation}
where $u = u(x, y, t)$ and $v = v(x, y, t)$ represent the velocities in the horizontal and vertical directions, respectively, and $h = h(x, y, t)$ denotes the water depth. The term $b = b(x, y)$ stands for the spatially varying bathymetry, and $g_r$ is the gravitational acceleration.

The dataset simulates a 2D radial dam-break scenario within a square domain $\Omega = [-2.5, 2.5]^2$ over the time interval $t \in [0, 1]$. The initial condition is defined by:

\begin{equation}
h(t = 0, x, y) =
\begin{cases}
2.0, & \text{if } \sqrt{x^2 + y^2} < r, \\
1.0, & \text{if } \sqrt{x^2 + y^2} \geq r,
\end{cases}
\end{equation}
where the radius $r$ is randomly drawn from a distribution $D(0.3, 0.7)$.

Our objective is to approximate the solution operator $S$, defined as:
\begin{equation}
\begin{aligned}
S: h|_{t \in [0, t']} \mapsto h|_{t \in (t', T]}, \quad (x, y) \in \Omega,
\end{aligned}
\end{equation}
with $t' = 0.009\,\text{s}$ and $T = 1.000\,\text{s}$. Here, $h = h(x, y, t)$ represents the water depth over time.

Each sample in the dataset is discretised on a spatial grid of $128^2$ points and a temporal grid of 101 time steps. The first 10 time steps are used as input to the model, while the remaining 91 time steps serve as the target output. Following the protocol established by PDEBench\cite{takamoto2022pdebench}, the dataset consists of 900 samples for training and 100 samples for testing.

\paragraph{\fcircle[fill=deblue]{2pt} Diffusion Reaction.} 
The Diffusion Reaction equations are expressed as:
\begin{equation}
\begin{aligned}
\partial_t u &= D_u \partial_{xx} u + D_u \partial_{yy} u + R_u, \\
\partial_t v &= D_v \partial_{xx} v + D_v \partial_{yy} v + R_v,
\end{aligned}
\end{equation}
where the activator and inhibitor are represented by the functions \( u = u(x, y, t) \) and \( v = v(x, y, t) \). In addition, these two variables are non-linearly coupled variables. These functions describe the interaction between the activator and inhibitor in the system. 
\( D_u = 1 \times 10^{-3} \) and \( D_v = 5 \times 10^{-3} \) are the diffusion coefficients for the activator and inhibitor, respectively.

The reaction terms for the activator and inhibitor are then defined as follows:
\begin{equation}
\begin{aligned}
R_u(u, v) = u - u^3 - k - v, \quad R_v(u, v) = u - v,
\end{aligned}
\end{equation}
with \( k = 5 \times 10^{-3} \).
The simulation is performed over the domain \( \Omega = [-1, 1]^2 \) with the time interval \( t \in [0, 5] \). The solution operator \( S \) is defined as:
\begin{equation}
\begin{aligned}
S: \{u, v\}_{t \in [0, t']} \mapsto \{u, v\}_{t \in (t', T]}, \quad (x, y) \in \Omega,
\end{aligned}
\end{equation}
where \( t' = 0.045 \, \text{s} \) and \( T = 5.000 \, \text{s} \), and the spatial domain is \( \Omega = [-1, 1]^2 \). Here, \( u = u(x, y, t) \) and \( v = v(x, y, t) \) represent the activator and inhibitor, respectively. In this dataset, we follow the same discretization scheme similar to the Shallow Water equation, where each sample is downsampled to a spatial resolution of \( 128^2 \) and a temporal resolution of 101 time steps (with 10 for input and rest of the 91 for target). Similar as the PDEBench protocol\cite{takamoto2022pdebench}, the dataset includes 900 samples for training and 100 samples for testing.

\paragraph{\fcircle[fill=deblue]{2pt} 2D Compressible Navier-Stokes equation.} 
The compressible fluid dynamics equations govern fluid flows, expressed as:  \begin{align}
&\partial_t \rho + \nabla \cdot (\rho \mathbf{v}) = 0, \label{eq:CNS_mass} \\
&\rho \left( \partial_t \mathbf{v} + \mathbf{v} \cdot \nabla \mathbf{v} \right) 
= -\nabla p + \eta \Delta \mathbf{v} + \left( \zeta + \frac{\eta}{3} \right) \nabla (\nabla \cdot \mathbf{v}), \label{eq:CNS_momentum} \\
&\partial_t \left( \epsilon + \frac{\rho v^2}{2} \right) 
+ \nabla \cdot \left[ \left( p + \epsilon + \frac{\rho v^2}{2} \right) \mathbf{v} - \mathbf{v} \cdot \boldsymbol{\sigma}' \right] = 0, \label{eq:CNS_energy}
\end{align}
where $\rho$ is the mass density, $\mathbf{v}$ denotes the fluid velocity, $p$ is the thermodynamic pressure, and $\epsilon$ represents the internal energy density determined through an equation of state. The tensor $\boldsymbol{\sigma}'$ is the viscous stress tensor, while $\eta$ and $\zeta$ correspond to the shear and bulk viscosity coefficients. This equation can capture complex phenomena, including the formation and propagation of shock waves, and finds applications in diverse real‐world scenarios such as aerodynamic flow around aircraft wings and the dynamics of interstellar gases. For more details of the dataset, we refer to \cite{takamoto2022pdebench} Appendix D.

\paragraph{\fcircle[fill=deblue]{2pt} NS2D-c.} This problem involves a 2D steady-state Navier--Stokes system posed on a rectangular domain with four circular holes. Specifically, the domain is  
$\Omega = [0,8]^2 \setminus \bigcup_{i=1}^4 R_i$,
where \( R_i \) denotes the \(i\)-th circular region. The governing equations are  
\begin{align}
(u \cdot \nabla) u &= \frac{1}{\mathrm{Re}} \nabla^2 u - \nabla p, \tag{14} \\
\nabla \cdot u &= 0, \tag{15}
\end{align}
where \(u\) is the velocity field, \(p\) is the pressure, and \(\mathrm{Re}\) is the Reynolds number. At the outlet, the pressure is fixed at zero, while at the inlet, the horizontal velocity component is prescribed as $u_x = \frac{y(8-y)}{16}$.

\textbf{Implementation \& Evaluation Protocol.} 
As Transformers have become the go-to architecture for PDE modelling and serve as the primary counterpart to SSM models, we selected three state-of-the-art Transformers as our baselines: GNOT\cite{hao2023gnot}, Galerkin Transformer (G.T.)\cite{cao2021choose}, and OFormer\cite{li2022transformer}. In addition, we also included several non-transformer-based baselines: FNO \cite{li2020fourier} , UNet \cite{ronneberger2015u}, DeepONet \cite{lu2019deeponet} 
To achieve a fair comparison between Transformers and Mamba, we integrated the S6 block and Cross S6 block to replace self-attention and cross-attention in each model, creating modified versions of the original architectures. 
All three experimental methods initially adopt a linear attention mechanism as described in their original publications, while we evaluated two configurations for each of them: an implementation with standard softmax attention mechanism (w/S.A.) and a Mamba-enhanced implementation (our Mamba Neural Operator principle) (w/Mamba). 
All experiments were conducted on a single NVIDIA RTX 4090 GPU with 24GB of memory to ensure consistent and fair comparison conditions.
Three metrics including Root Mean Squared Error (RMSE), Normalised RMSE (nRMSE), and Relative L2 Norm (RL2) were utilised for evaluation. For the testing time experiments, we take the average of 100 run on NVIDIA RTX 4090 GPU with 24GB of memory and Intel Xeon w7-3455.

\begin{table*}[t]
\caption{
Quantitative comparison on Darcy Flow ($\beta=1$) across three methods with linear attention (original version), softmax attention and Mamba.
The performance is measured in terms of Root Mean Squared Error (RMSE), Normalised RMSE (nRMSE), and Relative L2 Norm (RL2), with the best-performing results highlighted.
}
\centering
\resizebox{0.7\textwidth}{!}{
\begin{tabular}{c>{\columncolor[HTML]{FFFFFF}}c|>{\columncolor[HTML]{FFFFFF}}c>{\columncolor[HTML]{FFFFFF}}c>{\columncolor[HTML]{FFFFFF}}c}
\hline
\cellcolor[HTML]{EFEFEF} & \cellcolor[HTML]{EFEFEF} & \multicolumn{3}{c}{\cellcolor[HTML]{EFEFEF}\textsc{DarcyFlow}} \\ \cline{3-5} 
\multirow{-2}{*}{\cellcolor[HTML]{EFEFEF}\textsc{Method}} & \multirow{-2}{*}{\cellcolor[HTML]{EFEFEF}\textsc{Type}} & RMSE$\downarrow$ & nRMSE$\downarrow$ & RL2$\downarrow$ \\ \hline
\cellcolor[HTML]{FFFFFF} FNO  (Li et al.2020)  & Non-Transformer & 0.0119 & 0.0642 & 0.0363 \\
\cellcolor[HTML]{FFFFFF} UNet (Ronneberger et al.2015)  & Non-Transformer & 0.0064 & 0.0325 & 0.0196 \\
\cellcolor[HTML]{FFFFFF} DeepONet  (Lu et al.2019)  & Non-Transformer & 0.0193 & 0.1023 & 0.1023 \\ \hline
\cellcolor[HTML]{FFFFFF} GNOT (Hao et al, 2023) & Galerkin & 0.0070 & 0.0485 & 0.0370 \\
\cellcolor[HTML]{FFFFFF} w/S.A.  & Softmax & 0.0061 & 0.0394 & 0.0299 \\
\cellcolor[HTML]{FFFFFF} w/Mamba (MNO)  & Mamba & \cellcolor[HTML]{D9FFD9}0.0061 & \cellcolor[HTML]{D9FFD9}0.0367 & \cellcolor[HTML]{D9FFD9}0.0297 \\ \hline
\cellcolor[HTML]{FFFFFF} G.T. (Cao et al, 2021)& Galerkin & 0.0188 & 0.2027 & 0.1261 \\
\cellcolor[HTML]{FFFFFF} w/S.A. & Softmax & 0.0103 & 0.1050 & 0.0648 \\
\cellcolor[HTML]{FFFFFF} w/Mamba (MNO)   & Mamba & \cellcolor[HTML]{D9FFD9}0.0059 & \cellcolor[HTML]{D9FFD9}0.0374 & \cellcolor[HTML]{D9FFD9}0.0277 \\ \hline
\cellcolor[HTML]{FFFFFF} OFormer (Li et al, 2023) & Normalised & 0.0054 & 0.0253 & 0.0242 \\
\cellcolor[HTML]{FFFFFF} w/S.A. & Softmax & 0.0066 & 0.0324 & 0.0323 \\
\cellcolor[HTML]{FFFFFF} w/Mamba (MNO) & Mamba & \cellcolor[HTML]{D9FFD9}0.0054 & \cellcolor[HTML]{D9FFD9}0.0244 & \cellcolor[HTML]{D9FFD9}0.0241 \\ \hline
\end{tabular}
}
\label{Dacyflow}
\end{table*}

\begin{table*}[t]
\caption{
Quantitative comparisons on Shallow Water 2D (SW2D) and Diffusion Reaction 2D (DR2D) across three methods with linear attention (original version) and Mamba.
The performance is measured in terms of Root Mean Squared Error (RMSE), Normalised RMSE (nRMSE), and Relative L2 Norm (RL2), with the best‐performing results highlighted in green.
}
\centering
\resizebox{0.9\textwidth}{!}{%
\begin{tabular}{c
>{\columncolor[HTML]{FFFFFF}}c |
>{\columncolor[HTML]{FFFFFF}}c 
>{\columncolor[HTML]{FFFFFF}}c 
>{\columncolor[HTML]{FFFFFF}}c |
>{\columncolor[HTML]{FFFFFF}}c 
>{\columncolor[HTML]{FFFFFF}}c 
>{\columncolor[HTML]{FFFFFF}}c |
>{\columncolor[HTML]{FFFFFF}}c 
>{\columncolor[HTML]{FFFFFF}}c 
>{\columncolor[HTML]{FFFFFF}}c }
\hline
\cellcolor[HTML]{EFEFEF} & \cellcolor[HTML]{EFEFEF} & \multicolumn{3}{c|}{\cellcolor[HTML]{EFEFEF}SW2D} & \multicolumn{3}{c|}{\cellcolor[HTML]{EFEFEF}DR2D} & \multicolumn{3}{c}{\cellcolor[HTML]{EFEFEF}CFD2D} \\ \cline{3-11} 
\multirow{-2}{*}{\cellcolor[HTML]{EFEFEF}\textsc{Method}} & \multirow{-2}{*}{\cellcolor[HTML]{EFEFEF}\textsc{Type}} & RMSE$\downarrow$ & nRMSE$\downarrow$ & RL2$\downarrow$ & RMSE$\downarrow$ & nRMSE$\downarrow$ & RL2$\downarrow$ & RMSE$\downarrow$ & nRMSE$\downarrow$ & RL2$\downarrow$ \\ \hline
\cellcolor[HTML]{FFFFFF}FNO & Non-Transformer & 0.0046 & 0.0044 & 0.0048 & 0.0081 & 0.1203 & 0.0715 & - & - & - \\
\cellcolor[HTML]{FFFFFF}UNet & Non-Transformer & 0.0858 & 0.0824 & 0.0914 & 0.0615 & 0.8449 & 0.5614 & - & - & - \\
\cellcolor[HTML]{FFFFFF}DeepONet & Non-Transformer & 0.0027 & 0.0026 & 0.0027 & 0.0667 & 0.8771 & 0.6040 & - & - & - \\ \hline
\cellcolor[HTML]{FFFFFF}GNOT & Galerkin & 0.0026 & 0.0025 & 0.0027 & 0.0567 & 0.6953 & 0.7233 & 3.9873 & 0.2447 & 0.3457 \\
\cellcolor[HTML]{FFFFFF} MNO & Mamba & \cellcolor[HTML]{D9FFD9}0.0023 & \cellcolor[HTML]{D9FFD9}0.0022 & \cellcolor[HTML]{D9FFD9}0.0024 & \cellcolor[HTML]{D9FFD9}0.0060 & \cellcolor[HTML]{D9FFD9}0.0811 & \cellcolor[HTML]{D9FFD9}0.0570 & \cellcolor[HTML]{D9FFD9}3.5681 & \cellcolor[HTML]{D9FFD9}0.2309 & \cellcolor[HTML]{D9FFD9}0.3113 \\ \hline
\cellcolor[HTML]{FFFFFF}G.T. & Galerkin & 0.0037 & 0.0035 & 0.0038 & 0.0083 & 0.1259 & 0.0723 & 4.4961 & 0.8806 & 0.6380 \\
\cellcolor[HTML]{FFFFFF} MNO & Mamba & \cellcolor[HTML]{D9FFD9}0.0013 & \cellcolor[HTML]{D9FFD9}0.0013 & \cellcolor[HTML]{D9FFD9}0.0014 & \cellcolor[HTML]{D9FFD9}0.0012 & \cellcolor[HTML]{D9FFD9}0.0183 & \cellcolor[HTML]{D9FFD9}0.0099 & \cellcolor[HTML]{D9FFD9}0.5014 & \cellcolor[HTML]{D9FFD9}0.4813 & \cellcolor[HTML]{D9FFD9}0.0736 \\ \hline
\cellcolor[HTML]{FFFFFF}OFormer & Normalised & \cellcolor[HTML]{D9FFD9}0.0020 & \cellcolor[HTML]{D9FFD9}0.0020 & \cellcolor[HTML]{D9FFD9}0.0021 & 0.0177 & 0.2681 & 0.1559 & 0.5181 & 0.5938 & 0.0973 \\
\cellcolor[HTML]{FFFFFF} MNO & Mamba & 0.0021 & 0.0021 & 0.0022 & \cellcolor[HTML]{D9FFD9}0.0123 & \cellcolor[HTML]{D9FFD9}0.1712 & \cellcolor[HTML]{D9FFD9}0.1134 & \cellcolor[HTML]{D9FFD9}0.5018 & \cellcolor[HTML]{D9FFD9}0.5897 & \cellcolor[HTML]{D9FFD9}0.0889 \\ \hline
\end{tabular}}
\label{SW2D_extended}
\end{table*}

\subsection{Choose Your Winner: Transformer vs. Mamba for PDEs}
We begin by evaluating the performance of Transformers, their variants, and Mamba on the Darcy Flow dataset, as presented in Table~\ref{Dacyflow}. The results demonstrate that incorporating Mamba consistently improves performance across all metrics and models. Among the non-Transformer baselines (FNO, UNet, DeepONet), UNet is the strongest performer. Compared to this best baseline, our Mamba-based models still secure a further ~15.6\% reduction in RMSE and ~25\% reduction in nRMSE, highlighting Mamba’s advantage, even outperforming non-transformer methods. 
For GNOT, while the RMSE remains close, the nRMSE and RL2 values are reduced, indicating that Mamba effectively refines predictions. The G.T. sees the most significant enhancement, with the RMSE dropping by 40\% when Mamba is used. This suggests that Mamba's design addresses the shortcomings of traditional Galerkin-type attention in capturing complex PDE dynamics.
For OFormer, Mamba not only retains the strong baseline performance but also achieves improvements across all metrics. The reduction in RL2 indicates that Mamba’s mechanism is better at mapping the solution space of PDEs with higher precision. Mamba also demonstrates an enhanced ability to capture the complex spatial correlations inherent to Darcy Flow more effectively.

\begin{figure}[]
    \centering 
    \includegraphics[width=0.7\linewidth]{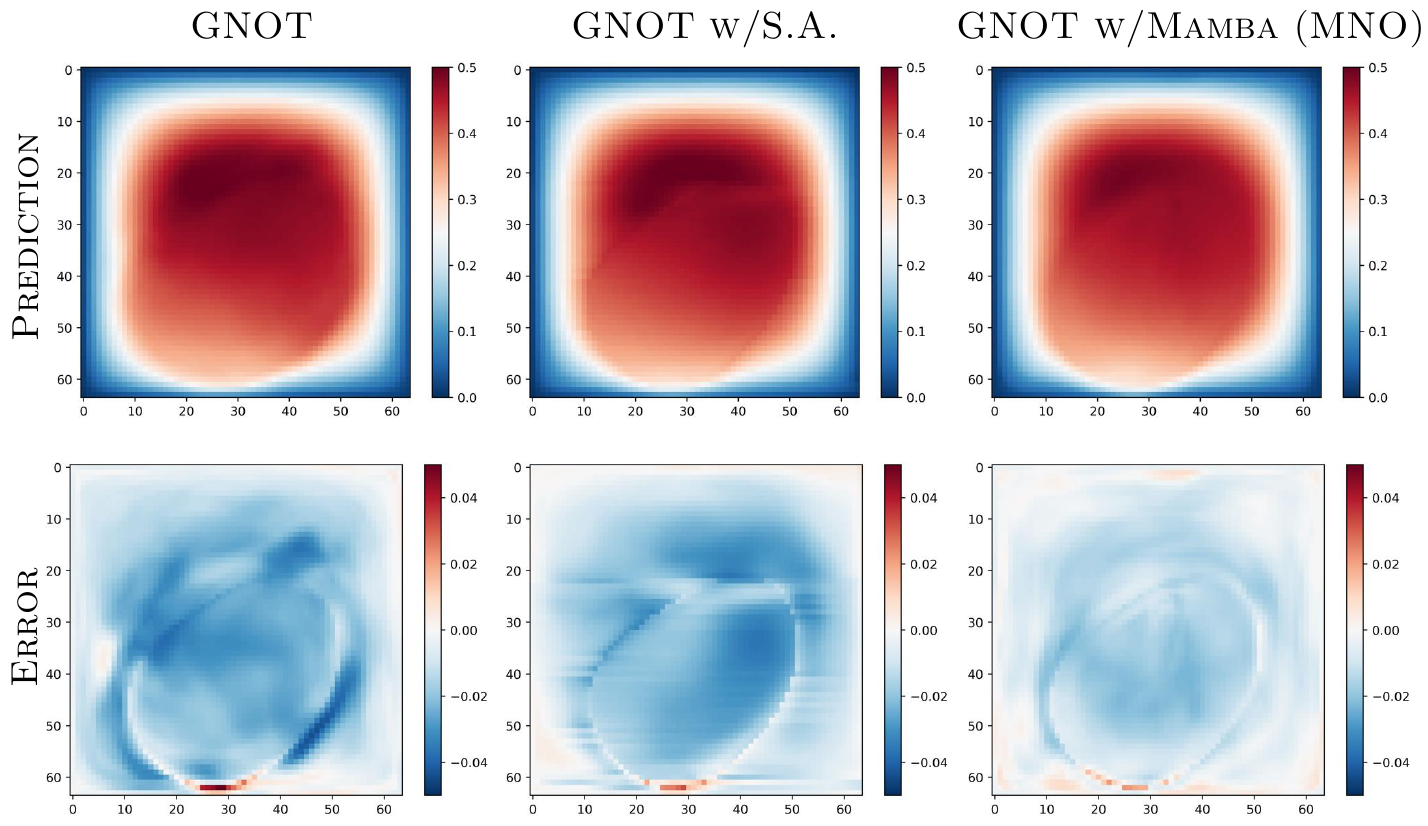}
    \caption{Results of prediction map and error map of the GNOT across three versions: Galerkin attention, Softmax attention, and Mamba.}   
    \label{fig:visMap}
\end{figure}

\begin{figure*}[t!]
    \centering
    \includegraphics[width=\linewidth]{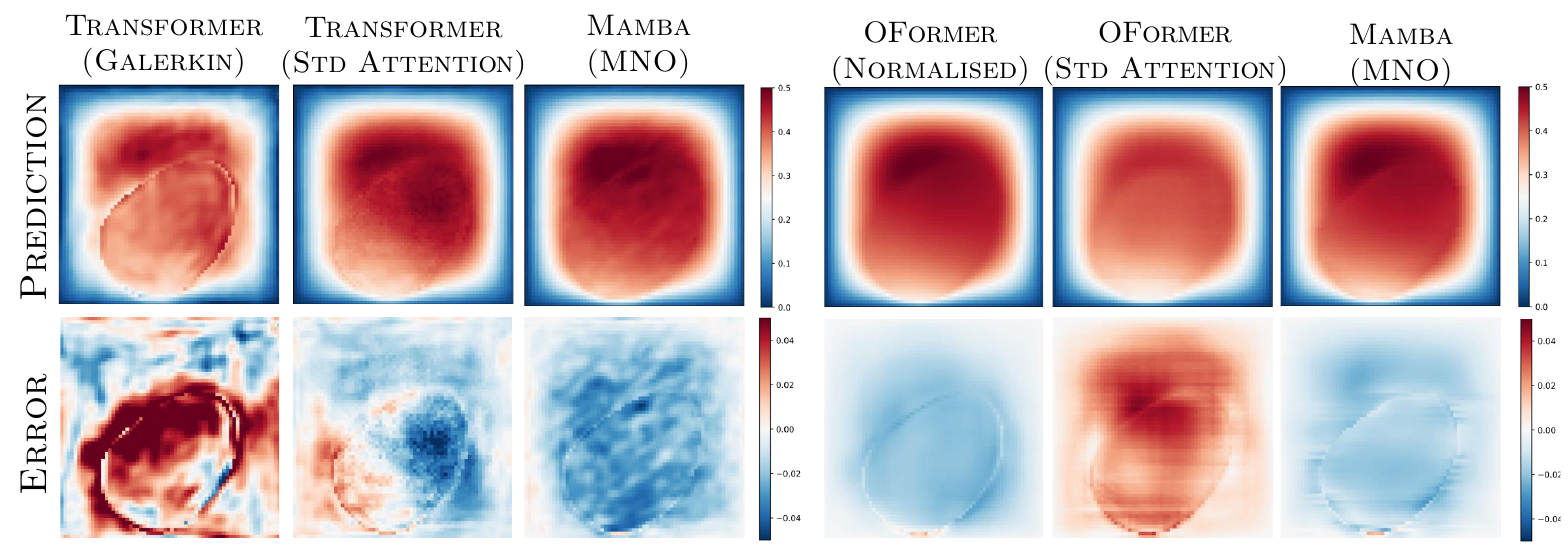}
    \caption{
    Results of prediction map and error map of the Galerkin Transformer and OFormer across three versions: Galerkin attention, Softmax attention, and Mamba.
    }
    \label{fig:v2}
\end{figure*}

\begin{figure*}[t!]
    \centering
    \includegraphics[width=\linewidth]{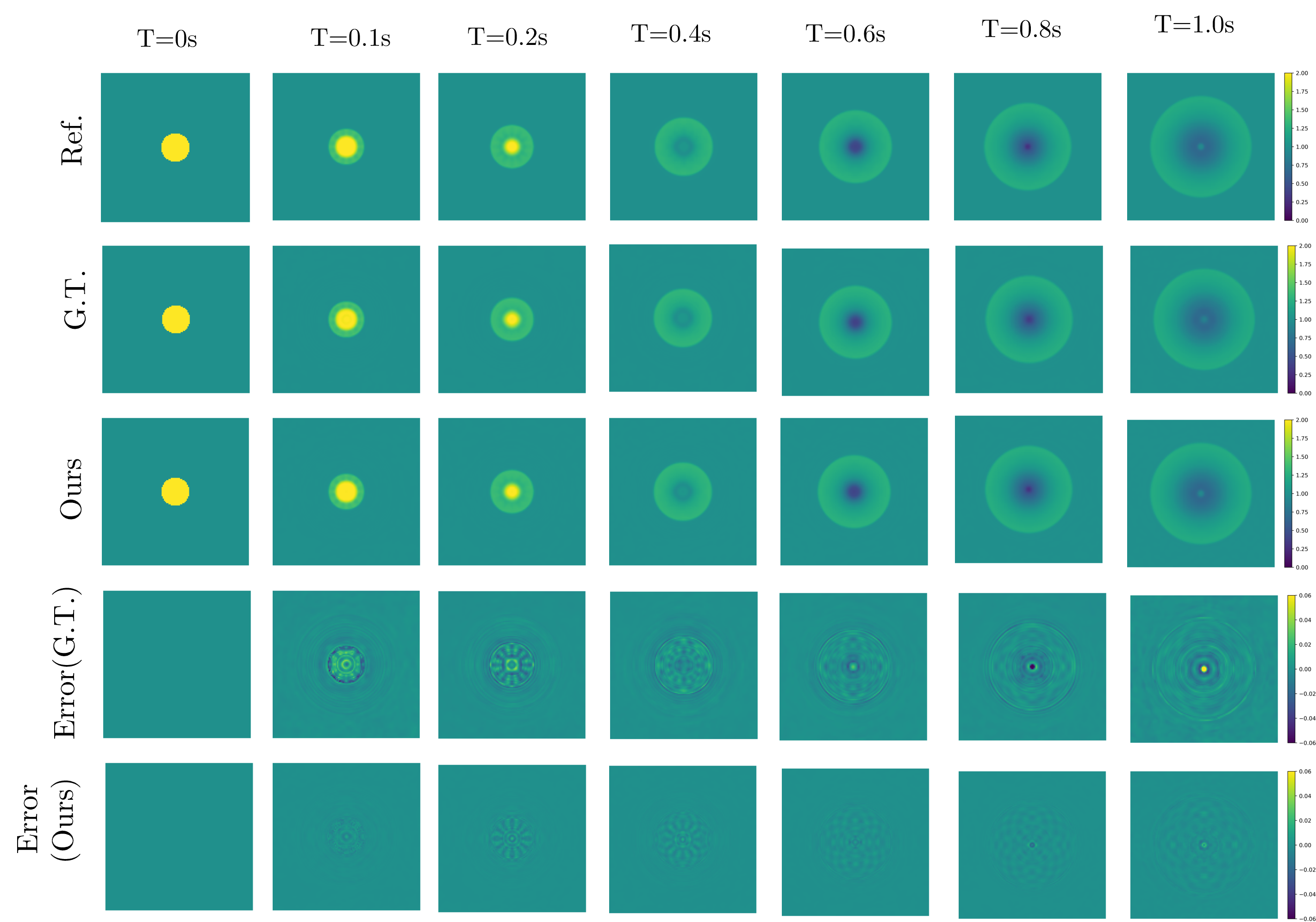}
    \caption{
    Visualized predictions on the Shallow Water dataset using the original Galerkin Transformer (G.T.) and our Mamba-enhanced version. The first row shows the ground truth, the second and third rows display the model predictions, and the final two rows present the corresponding absolute contour errors.
}
    \label{fig:v3}
\end{figure*}

\begin{figure*}[t!]
    \centering
    \includegraphics[width=\linewidth]{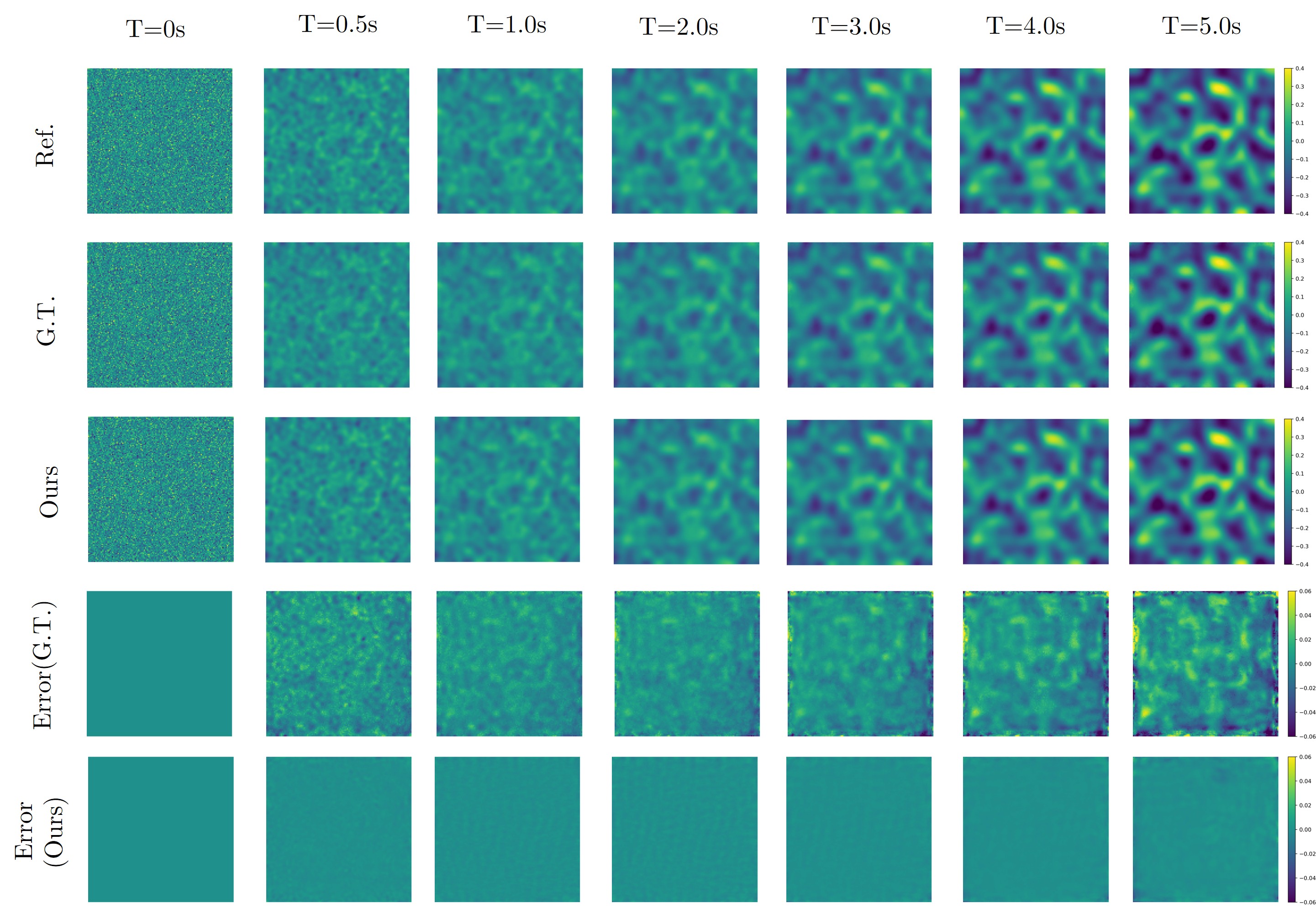}
    \caption{
        Visualized predictions on the Diffusion Reaction dataset using the original Galerkin Transformer (G.T.) and our Mamba-enhanced version. The first row shows the ground truth, the second and third rows display the model predictions, and the final two rows present the corresponding absolute contour errors.
    }
    \label{fig:v4}
\end{figure*}

On the SW2D dataset, Mamba consistently outperforms the original Transformer models across all metrics. Among the non-Transformer baselines (FNO, UNet, DeepONet), DeepONet is the strongest performer. Relative to this best baseline, our Mamba-based models still secure a further ~51.9\% reduction in RMSE and ~50\% reduction in nRMSE.
GNOT with Mamba achieves a lower RMSE and RL2, demonstrating Mamba’s ability to capture complex flow dynamics. The G.T. shows the most significant improvement, with RMSE of 65\% reduction—highlighting Mamba’s superior capability in accurately representing the system’s behaviour. For OFormer, Mamba maintains comparable values but increases the RL2. On the DR2D dataset, the Galerkin-type Transformer (G.T.) equipped with Mamba (G.T. MNO) substantially outperforms all non-Transformer-based baselines. Compared to the best non-Transformer model, FNO, the G.T. MNO achieves a 85.2\% reduction in RMSE (from 0.0081 to 0.0012), alongside similar proportional improvements in nRMSE and RL2.
For Galerkin-based models, leveraging Mamba into GNOT yields a notable reduction in error, achieving a ~89.4\% decrease in RMSE and similar proportional improvements in nRMSE and RL2 compared to the original GNOT. The Galerkin Transformer (G.T.) also benefit a improvement , with Mamba reducing RMSE from 0.0083 to 0.0012, and RL2 from 0.0723 to 0.0099, marking a more than 80\% improvement in RL2. For normalised attention models such as OFormer, Mamba further enhances accuracy, lowering RMSE from 0.0177 to 0.0123 and RL2 from 0.1559 to 0.1134. These results demonstrate that Mamba’s state-space formulation effectively captures the nonlinear spatio-temporal dependencies in DR2D, leading to both higher accuracy and improved stability over all tested architectures.

On the most challenging benchmark, CFD2D, trained at an exceptionally high spatial resolution of 512 $\times$ 512, our MNO consistently delivers the lowest errors and establishes a new state of the art. Comparing Transformer-based baselines—GNOT, a Galerkin Transformer (G.T.), and the OFormer—struggle to accurately capture the highly non-linear, vortical flow fields at this resolution. Re-implementing each of these architectures with the Mamba and bi-directional scan reduces their RMSE across the board: from 3.99 to 3.57 for GNOT (-10\%), from 4.50 to 0.50 for the Galarkin Transformer (-89\%), and from 0.518 to 0.502 for OFormer (-3\%). Similar reductions appear in nRMSE and RL2, with the most dramatic improvement being an 88\% reduction in RL2 for the G.T. baseline. The case of GNOT + Mamba, where the model achieves the lowest nRMSE but the highest RL2, indeed suggests that while the model performs well on average, it may struggle with capturing certain large-scale features. This discrepancy can arise from the fact that nRMSE tends to focus on relative error, which can mask larger deviations in regions with complex dynamics, while RL2 captures absolute errors across the domain, highlighting discrepancies in larger-scale structures. Importantly, at this high resolution, our MNO also surpasses all Transformer-based models, demonstrating its ability to scale effectively and retain superior accuracy even when solving computationally intensive PDEs on small grids. This underscores MNO’s capability to tackle highly challenging PDE tasks at high resolution.

The results across all datasets demonstrate a clear advantage of the Mamba Neural Operator over Transformer architectures for PDEs. While Transformers are effective at capturing dependencies and patterns, Mamba’s specialised attention mechanisms provide a more understanding of the complex dynamics involved. By leveraging its unique cross-attention and self-attention blocks, Mamba not only achieves lower error rates but also enhances the stability and precision of predictions, particularly in highly nonlinear systems. \textit{These results suggest that Mamba enhances the expressive power and accuracy of neural operators, indicating that it is not just a complement to Transformers but a superior framework for PDE-related tasks}, bridging the gap between efficient representation and accurate solution approximation.

We further validate Mamba’s potential through visualisations, as shown in Figure~\ref{fig:visMap}. The prediction and error maps reveal that Mamba consistently outperforms all Transformer variants, delivering more accurate solutions with lower error across challenging regions. Mamba handles fine details, particularly in capturing sharp gradients and subtle variations that standard attention mechanisms often miss. Compared to the Galerkin and Softmax attention Transformer models, Mamba reduces error propagation and improves spatial coherence. 

Figure~\ref{fig:v2} presents the prediction and error maps for the Galerkin Transformer (G.T.) and OFormer across three configurations: Galerkin attention, standard softmax attention, and Mamba (MNO). The results show that Mamba consistently achieves lower prediction errors, especially in regions with high variability, highlighting its ability to capture complex dynamics with greater precision compared to other configurations.

Figure~\ref{fig:v3} and Figure~\ref{fig:v4} provide visualised predictions over time for the Shallow Water and Diffusion Reaction datasets, respectively, using the original Galerkin Transformer and its Mamba-enhanced version. For the Shallow Water dataset, the Mamba-integrated model better preserves fine details and the circular wavefronts as time progresses, reflecting its superior capability to maintain spatial coherence. Similarly, in the Diffusion Reaction dataset, Mamba reduces the spread of error and better approximates the reference solution, demonstrating improved stability and generalisation in long-term simulations.

\begin{figure*}[t!]
    \centering
    \includegraphics[width=\linewidth]{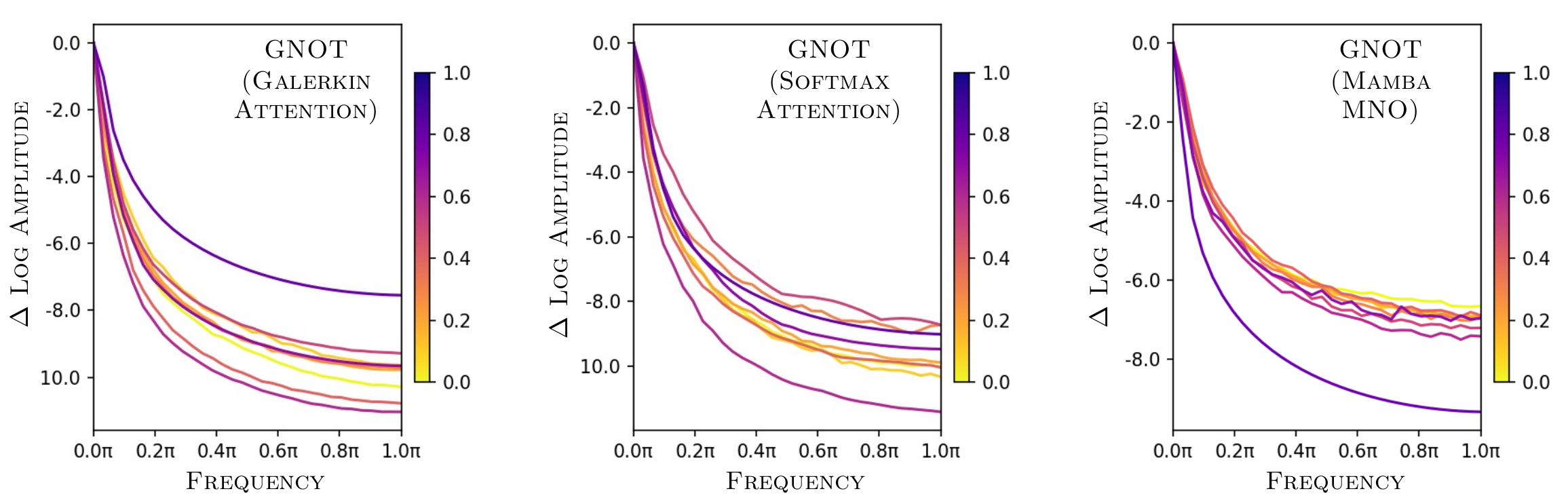}
    \caption{
    Fourier analysis comparing three GNOT versions: Galerkin attention, Softmax attention, and Mamba. The $\Delta$ log amplitude shows how each model handles frequency components. We calculate the change by comparing the log amplitude at the center (0.0 $\pi$) and boundary frequencies (1.0 $\pi$). For clarity, only half-diagonal components of the 2D Fourier-transformed feature maps are shown.}
    
    \label{fig:fu1}
\end{figure*}
\begin{figure*}[t!]
    \centering
    \includegraphics[width=\linewidth]{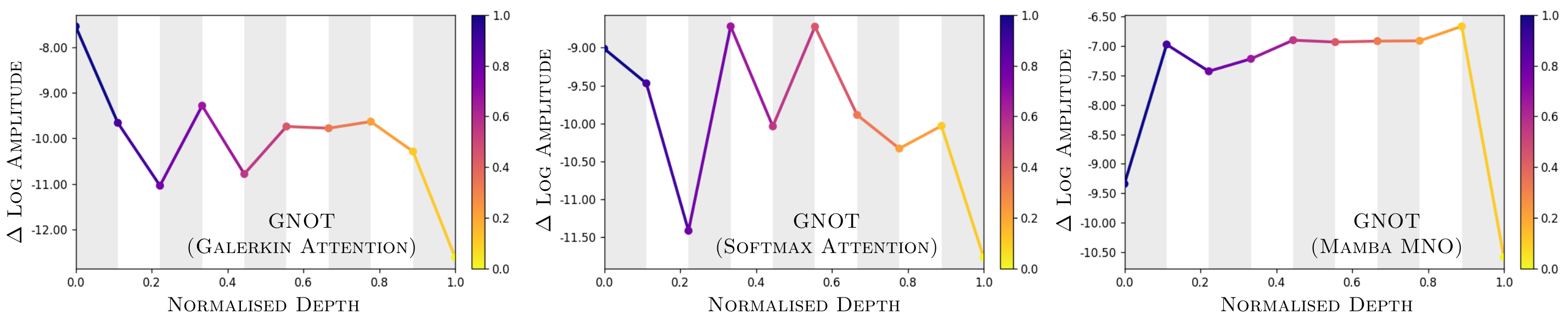}
    \caption{$\Delta$ log amplitudes for Galerkin attention, Softmax attention, and Mamba. Gray regions indicate the operator, and white regions show MLP. Mamba shows a more stable response across frequencies.}
    \label{fig:fu2}
\end{figure*}

\subsection{Why the Winner Wins: Breaking Down Mamba’s Win}
We aim to explore why Mamba outperforms Transformers by examining the frequency response of feature maps. This analysis helps us understand how each model handles high-frequency signals and evaluate its ability to maintain stability and robustness.
The results in Figure~\ref{fig:fu1} compare the frequency response of three GNOT variants: Galerkin attention, Softmax attention, and Mamba. The Galerkin version shows a sharp decline in high-frequency components, indicating underfitting and loss of fine details. The Softmax version retains more high frequencies but risks instability and noise sensitivity. Mamba, on the other hand, demonstrates a balanced suppression of high-frequency signals, maintaining stability and robustness. The change in log amplitude across the frequency range is more uniform, indicating that Mamba effectively balances between capturing necessary high-frequency information and filtering out noise.
This controlled response across the spectrum highlights why Mamba is better suited for PDEs.

Figure~\ref{fig:fu2} shows the $\Delta$ log magnitudes across the normalised depth for the Galerkin attention, Softmax attention, and Mamba versions of GNOT. The Galerkin and Softmax versions exhibit sharp fluctuations, indicating instability and inconsistent feature extraction at different normalised depths. In contrast, Mamba maintains a steady and flat profile, reflecting robust and stable feature extraction. The gray and white bands indicate the alternating roles of the operator and NLP components, further emphasising Mamba’s balanced performance across layers, making it ideal for handling complex PDEs.

\begin{table}[t!]
\centering
\begin{minipage}{0.44\textwidth}
\centering
\caption{Comparisons with different query positions \\using nRMSE.
}
\begin{tabular}{c|c|c}
\hline
\cellcolor[HTML]{EFEFEF}\textsc{Method} & \multicolumn{2}{c}{\cellcolor[HTML]{EFEFEF}Query Positions} \\ \cline{2-3} 
\cellcolor[HTML]{EFEFEF}\textsc{} & \cellcolor[HTML]{EFEFEF}Identical & \cellcolor[HTML]{EFEFEF}Diagonal \\ \hline
\cellcolor[HTML]{FFFFFF} OFormer & 0.0253 & 0.0318 \\
\cellcolor[HTML]{FFFFFF} w/S.A. & 0.0324 & 0.0382 \\ 
\cellcolor[HTML]{FFFFFF} w/Mamba  & \cellcolor[HTML]{D9FFD9}0.0244 & \cellcolor[HTML]{D9FFD9}0.0314 \\ \hline
\end{tabular}
\label{diagonal}
\end{minipage}%
\hfill
\begin{minipage}{0.56\textwidth}
\centering
\caption{
Comparisons with different dataset sizes using nRMSE.
}
\begin{tabular}{c|c|c|c|c}
\hline
\cellcolor[HTML]{EFEFEF}\textsc{Method} & \multicolumn{4}{c}{\cellcolor[HTML]{EFEFEF}Dataset Sizes} \\ \cline{2-5} 
\cellcolor[HTML]{EFEFEF}\textsc{} & \cellcolor[HTML]{EFEFEF}\textsc{9K } & \cellcolor[HTML]{EFEFEF}\textsc{5K} & \cellcolor[HTML]{EFEFEF}\textsc{2K} & \cellcolor[HTML]{EFEFEF}\textsc{1K} \\ \hline
\cellcolor[HTML]{FFFFFF} GNOT & 0.0485 & 0.0567 & 0.0777 & 0.1174 \\
\cellcolor[HTML]{FFFFFF} w/S.A. & 0.0394 & 0.0400 & 0.0526 & 0.0776 \\
\cellcolor[HTML]{FFFFFF} w/Mamba & \cellcolor[HTML]{D9FFD9}0.0367 & \cellcolor[HTML]{D9FFD9}0.0376 & \cellcolor[HTML]{D9FFD9}0.0481 & \cellcolor[HTML]{D9FFD9}0.0617 \\ \hline
\end{tabular}
\label{datasize}
\end{minipage}
\end{table}

\begin{table}
\caption{Comparisons on query positions and dataset sizes using nRMSE.}{
\begin{tabular}{lcc|cccc}
\toprule
\textsc{Method} 
& \multicolumn{2}{c}{(a) Query Positions} 
& \multicolumn{4}{c}{(b) Dataset Sizes} \\ 
\cmidrule(lr){2-3} \cmidrule(lr){4-7}
& Identical & Diagonal 
& \textsc{9K} & \textsc{5K} & \textsc{2K} & \textsc{1K} \\ 
\midrule
OFormer     & 0.0253 & 0.0318 & --      & --      & --      & --      \\
w/S.A.      & 0.0324 & 0.0382 & 0.0394 & 0.0400 & 0.0526 & 0.0776 \\
w/Mamba     & \textbf{0.0244} & \textbf{0.0314} 
            & \textbf{0.0367} & \textbf{0.0376} & \textbf{0.0481} & \textbf{0.0617} \\
GNOT        & --     & --     & 0.0485 & 0.0567 & 0.0777 & 0.1174 \\
\bottomrule
\end{tabular}}
\label{diagonal}
\end{table}

\begin{table*}[t]
\caption{Quantitative comparisons on Darcy Flow across three Transformer-based architectures with Galarkin attention, softmax attention, and Mamba Neural Operator (MNO). Metrics include FLOPs, parameter count, inference time, and GPU memory usage}
\centering
\resizebox{1.0\textwidth}{!}{
\begin{tabular}{c>{\columncolor[HTML]{FFFFFF}}c|>{\columncolor[HTML]{FFFFFF}}c>{\columncolor[HTML]{FFFFFF}}c>{\columncolor[HTML]{FFFFFF}}c>{\columncolor[HTML]{FFFFFF}}c}
\hline
\cellcolor[HTML]{EFEFEF} & \cellcolor[HTML]{EFEFEF} & \multicolumn{3}{c}{\cellcolor[HTML]{EFEFEF}\textsc{DarcyFlow}} & \cellcolor[HTML]{EFEFEF} \\ \cline{3-6}
\multirow{-2}{*}{\cellcolor[HTML]{EFEFEF}\textsc{Method}} & \multirow{-2}{*}{\cellcolor[HTML]{EFEFEF}\textsc{Type}} & FLOPs(G) & Parameters (M) & Inference time (ms) & GPU Memory Usage (GIB) \\ \hline
\cellcolor[HTML]{FFFFFF} GNOT (Hao et al, 2023) & Galerkin & 29.71 & 5.59 & 12.04 & 0.15 \\
\cellcolor[HTML]{FFFFFF} w/S.A.  & Softmax & 235.26 & 5.59 & 162.93 & 8.12 \\
\cellcolor[HTML]{FFFFFF} w/Mamba (MNO)  & Mamba & 45.36 & 7.65 & 19.70 & 0.48 \\ \hline
\cellcolor[HTML]{FFFFFF} G.T. (Cao et al, 2021)& Galerkin & 85.59 & 26.24 & 15.85 & 0.27 \\
\cellcolor[HTML]{FFFFFF} w/S.A. & Softmax & 503.94 & 26.24 & 132.58 & 18.67 \\
\cellcolor[HTML]{FFFFFF} w/Mamba (MNO)   & Mamba & 116.99 & 33.45 & 32.56 & 0.78 \\ \hline
\cellcolor[HTML]{FFFFFF} OFormer (Li et al, 2023) & Normalised & 116.34 & 18.04 & 39.66 & 1.11 \\
\cellcolor[HTML]{FFFFFF} w/S.A. & Softmax & 977.14 & 19.10 & 127.64 & 4.83 \\
\cellcolor[HTML]{FFFFFF} w/Mamba (MNO) & Mamba & 60.13 & 10.89 & 22.89 & 1.13 \\ \hline
\end{tabular}
}
\label{Dacyflow_transformers}
\end{table*}

\subsection{Ablation Study: Final Battles, Winner Takes All}
\begin{figure*}[t!]
    \centering
    \includegraphics[width=\linewidth]{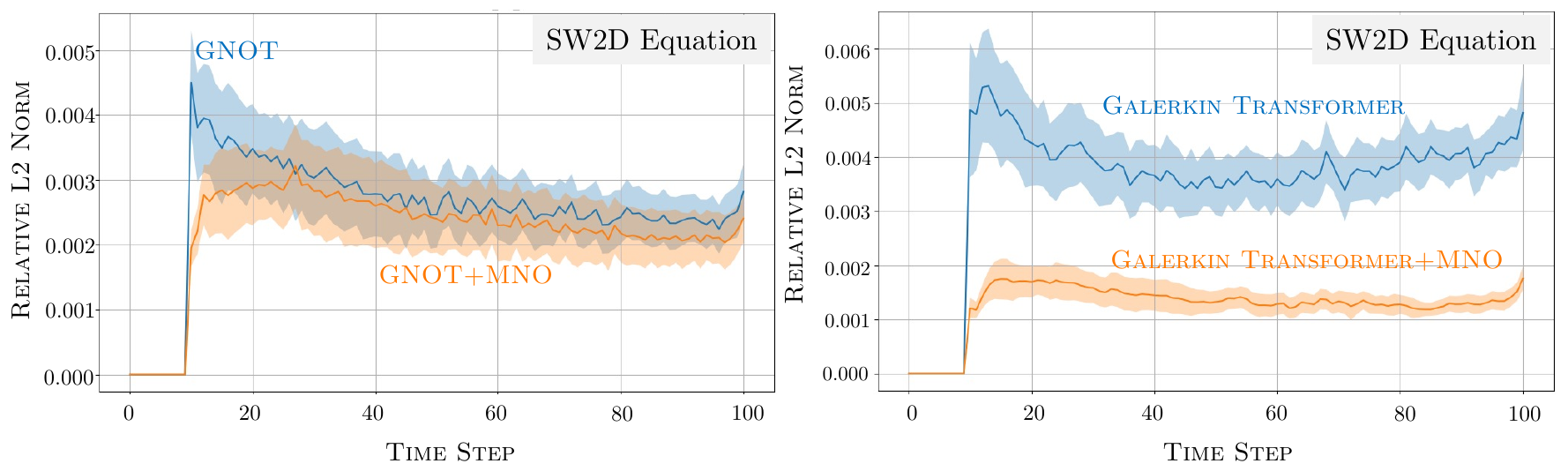}
    \caption{
    Relative L2 error on SW2D across time steps, comparing the performance of our Mamba Neural Operator (MNO) with the GNOT baseline and the Galerkin Transformer (G.T.). The left plot shows the evolution of relative L2 error over time steps for the SW2D equation, while the right plot shows the same for the SW2D equation. MNO consistently outperforms GNOT and G.T., particularly in terms of lower error growth, highlighting its superior ability to capture long-range dependencies and model complex dynamics more effectively.}
    
    \label{fig:fu2}
\end{figure*}

\begin{figure*}[t!]
    \centering
    \includegraphics[width=\linewidth]{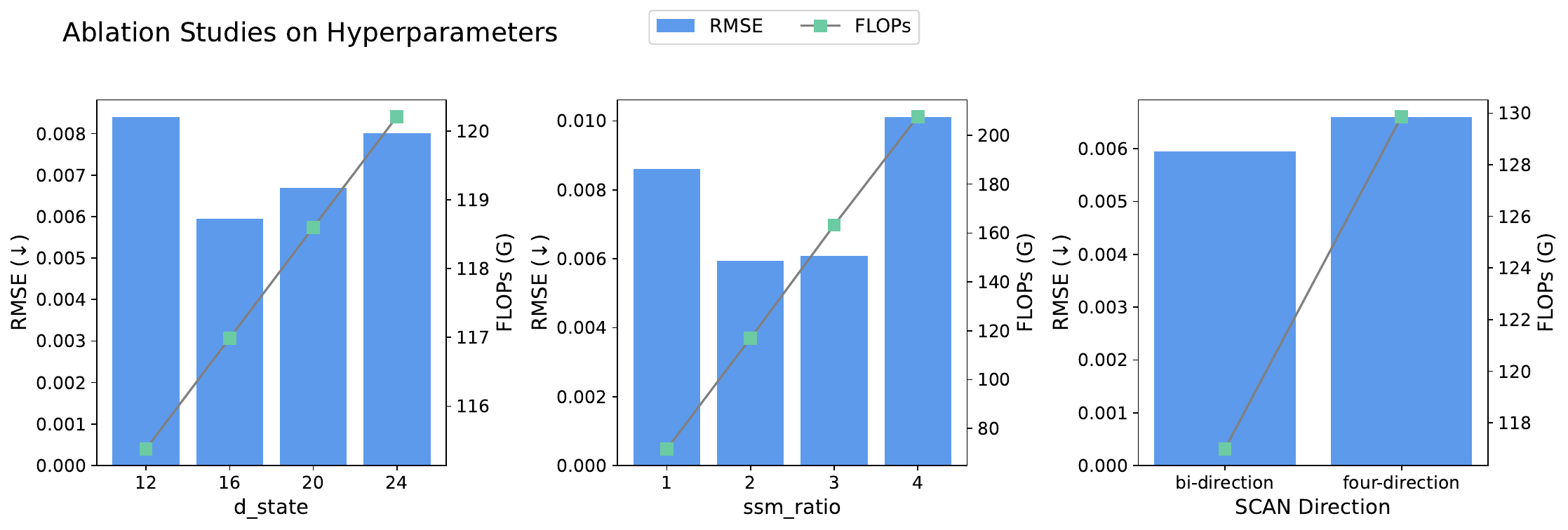}
    \caption{Ablation studies on $d\_state$(the hidden SSM state dimension), $ssm\_ratio$(the fraction of channels allocated to the SSM branch), and scan direction, showing trade-offs between RMSE (blue) and FLOPs (green).}
    \label{fig:abl}
\end{figure*}
\paragraph{Mamba vs. Transformers in Misalignment: A Battle for Query Positioning.}
Table~\ref{diagonal} compares nRMSE performance across two query scenarios: Identical positions (input and query points are the same) and Diagonal positions (shifted inputs creating a mismatch). Experiments are done using Darcy Flow. While prior work shows that Transformers handle inconsistent input-query positions well\cite{li2022transformer}, our results demonstrate a clear advantage of Mamba in both configurations. 
For Identical query positions, Mamba version achieves the lowest error, outperforming OFormer and its softmax variant, demonstrating \textit{Mamba’s superior ability to capture relationships when input and query points are perfectly aligned.}
For Diagonal query positions, where inputs and queries are misaligned, Mamba achieves the best performance compared to OFormer and its variant, demonstrating its superior \textit{ability to generalise under spatial shifts}.

\paragraph{Scaling Down Without Sacrifice: A Battle for Resilience with Limited Data.} 
Table~\ref{datasize} compares nRMSE performance across different dataset sizes for GNOT, GNOT with softmax attention (w/ S.A.), and GNOT with Mamba. Experiments are carried out using Darcy Flow. As dataset size decreases, Mamba consistently achieves the lowest error, demonstrating superior performance and robustness in data-scarce scenarios. For instance, with the smallest dataset (1K), Mamba achieves an nRMSE of 0.0617, significantly lower than GNOT’s and GNOT w/ S.A.’s, showcasing its resilience and generalisation capability even with limited data. \textit{This highlights Mamba’s efficiency in learning meaningful representations with fewer data points, making it a powerful choice for real-world applications where data availability is a constraint.}

\paragraph{Computational Efficiency.}  Across the three Transformer-based architectures—GNOT, Galerkin Transformer (G.T.), and OFormer—replacing the original attention mechanism with our Mamba Neural Operator (MNO) yields substantial efficiency gains. Compared to the softmax attention variants, MNO consistently reduces FLOPs, inference time, and GPU memory usage, often by an order of magnitude. For example, in OFormer, MNO lowers FLOPs from 977.14 G to just 60.13 G, reduces memory usage from 4.83 GiB to 1.13 GiB, and halves inference time. Similar trends are observed for GNOT and G.T., where MNO achieves up to a 10× reduction in FLOPs and over a 90 \% decrease in memory footprint. When compared with the normalized attention method, MNO delivers approximately 50 \% reductions in FLOPs, parameter count, and inference time, while maintaining similar GPU memory usage. Relative to the original Galerkin method, MNO exhibits higher FLOPs, a larger parameter count, and increases in both inference time and GPU memory usage. Theoretically, both Mamba and Galerkin methods have an O(N) memory cost; however, Mamba employs a bi-directional scan to enhance representational capacity, which in turn increases computational complexity relative to the Galerkin method. Despite this, Mamba consistently outperforms Galerkin in accuracy. Overall, these results demonstrate MNO’s ability to provide high‐accuracy PDE solutions with markedly improved computational efficiency, making it particularly well‐suited for large‐scale and real‐time applications.

\paragraph{Long-time stability.}  The long-time stability of the PDE solvers is clearly improved by integrating our Mamba Neural Operator (MNO) into both GNOT and the Galerkin Transformer architectures.  In our method, we adopts the one-shot prediction strategy, where the model predicts the entire future sequence directly from the initial condition:
\[
\{\hat{x}_1, \hat{x}_2, \ldots, \hat{x}_T\} = f_{\theta}(x_0).
\]

When training a one-shot prediction model, the network learns to capture correlations across the entire temporal sequence rather than progressing step-by-step. Because all future states are predicted jointly from the initial condition, the model can allocate its representational capacity to minimise the average error over the entire trajectory. In other words, it will not accumulate the error in each time step.  This joint optimisation may produce a characteristic pattern where the error is slightly higher at early times (reflecting transient or complex dynamics) and then decreases as the prediction converges toward smoother, easier-to-represent states. Conceptually, this process resembles performing a global least-squares fit across all time steps simultaneously, rather than executing a sequential simulation where errors accumulate step-by-step. In the standard GNOT and Galerkin Transformer, the relative L2 error exhibits a noticeable increase during early time steps and remains at a higher baseline throughout the prediction horizon.  In contrast, the MNO(GNOT+MNO and GalerkinTransformer+MNO) show a substantial reduction in error magnitude across all time steps, along with smaller error fluctuations. This reduced variance and lower error plateau indicate that the MNO-enhanced models maintain more accurate and stable predictions over extended time horizons, effectively mitigating error accumulation typical in autoregressive PDE forecasting. The effect is particularly pronounced in the Galerkin Transformer case, where MNO integration achieves over a two-fold reduction in long-term error levels.

\paragraph{Ablation Studies on Hyperparameter }  The ablation study in Figure~\ref{fig:abl} evaluates the impact of three key hyperparameters - $d\_state$, $ssm\_ratio$ and scan direction - on model performance (RMSE) and computational cost (FLOP). $d\_state$ means the dimensionality of the latent hidden state in the state-space model (SSM) branch of Mamba. A larger $d\_state$ means that the SSM maintains a richer internal memory representation at each step, potentially capturing more complex temporal dependencies, but at the cost of higher computational complexity. $ssm\_ratio$ means the proportion of the model’s total feature channels ($d_{\text{model}}$) that are allocated to the SSM branch within each Mamba block. For $d\_state$, increasing the hidden SSM state dimension from 12 to 24 leads to higher FLOPs and a non-monotonic effect on RMSE, with the best accuracy achieved at $d_{state} = 16 $. Similarly, raising the $ssm\_ratio$ from 1 to 4 steadily increases FLOPs, with the lowest RMSE observed at $ssm\_ratio =2$, beyond which accuracy degrades. Regarding scan direction, the four-directional variant results in slightly higher RMSE and increased FLOPs compared to the bi-directional scan, indicating that the added complexity does not translate to better accuracy in this setup. Overall, the results suggest that our model is the best choice.

\subsection{Limitation} 
\paragraph{Irregular Domain.} 
Our experiments focus on regular domains. To show encoder sensitivity on irregular geometries, we ran a  test on NS2D-c using a naïve mesh$\!\to\!$grid interpolation to reuse the 2D scan encoder. Under this preprocessing, GNOT attains RL2 = 0.0125 while MambaGNOT yields RL2 = 0.0304. This gap reflects information loss from interpolation rather than a structural limitation of MNO: cross-attention forms an explicit dense similarity matrix that is less sensitive to aliasing, whereas Cross-Mamba fuses sequences by driving structured state dynamics in linear time—beneficial on regular grids but not compensating for lost mesh connectivity. Extending MNO to irregular and 3D settings is straightforward at the encoder: replace the scan with mesh-aware tokenisation (e.g., Laplacian eigenfeatures, boundary masks, graph/mesh message passing, or local chart atlases with partition-of-unity) and apply the same MNO core; for 3D, use tri-directional scans or volumetric patching with the same $O(N)$ memory scaling. A full evaluation with mesh-aware encoders is beyond the current scope and is a clear avenue for future work.

\section{Conclusion}
We have introduced the concept of the Mamba Neural Operator (MNO), a framework that redefines how neural operators approach PDEs by integrating structured state-space models. Unlike closely related works, we formalise this connection by providing a theoretical understanding that demonstrates how neural operator layers share a comparable structural framework with time-varying SSMs, offering a fresh perspective on their underlying principles. Experimental results show that MNO significantly enhances the expressive power and accuracy of neural operators across various architectures and PDEs. This indicates that MNO is not merely a complement to Transformers, but a superior framework for PDE-related tasks, bridging the gap between efficient representation and precise solution approximation.

\section*{Acknowledgments}
CWC and AIAR acknowledge support from the Swiss National Science Foundation (SNSF) under grant number 20HW-1 220785. JH was supported in part by the Imperial College Bioengineering Department PhD Scholarship and the UKRI Future Leaders Fellowship(MR/V023799/1). GY was supported in part by the ERC IMI (101005122), the H2020 (952172), the MRC (MC/PC/21013), theRoyal Society (IEC\%NSFC\%211235), the NVIDIA Academic Hardware Grant Program, the SABER project supported by BoehringerIngelheim Ltd, NIHR Imperial Biomedical Research Centre (RDA01), The Wellcome Leap Dynamic resilience program (co-funded by Temasek Trust), UKRI guarantee funding for Horizon Europe MSCA Postdoctoral Fellowships (EP/Z002206/1), UKRI MRC ResearchGrant, TFS Research Grants (MR/U506710/1), Swiss National Science Foundation (Grant No. 220785), and the UKRI Future Leaders Fellowship (MR/V023799/1, UKRI2738). CBS acknowledges support from the Philip Leverhulme Prize, the Royal Society Wolfson Fellowship, the EPSRC advanced career fellowship EP/V029428/1, EPSRC grants EP/S026045/1 and EP/T003553/1, EP/N014588/1, EP/T017961/1, the Wellcome Innovator Awards 215733/Z/19/Z and 221633/Z/20/Z, CCMI, and the Alan Turing Institute. AIAR gratefully acknowledges the support of the Yau Mathematical Sciences Center, Tsinghua University. This work is also supported by the Tsinghua University Dushi Program.

\section*{Appendix A. Notation Table }
To improve readability and address any ambiguity, we summarise in Table~\ref{tab:notation} the notation used throughout the paper. The table groups symbols by context (Mamba Neural Operator, General PDE and hyperparameters) and provides a concise description of their meaning. This helps ensure consistency across sections and serves as a quick reference for the reader.

\section*{Appendix B. More explanation on the main method }

In this section, we provides an additional example on how the difference between Euler and ZOH method. Consider the scalar linear state--space model
\begin{equation}
{h}'(t) = a h(t) + b u(t),
\end{equation}
with $a=-1$, $b=2$. We look at one step of size $\Delta > 0$ from $t_k$ to $t_{k+1} = t_k + \Delta$.
For the Forward Euler, we defines as follows:
\begin{equation}
h_{k+1}^{\text{Euler}} = (1 + a \Delta)\, h_k + b \Delta\, u_k.
\end{equation}
For the Zero-Order Hold (ZOH), we defines as follows:
\begin{equation}
h_{k+1}^{\text{ZOH}} = e^{a\Delta}\, h_k +
\left( \int_0^\Delta e^{a\tau}\, d\tau \right) b u_k
= e^{a\Delta}\, h_k + \frac{e^{a\Delta}-1}{a}\, b u_k,
\end{equation}
which corresponds to the matrix pair
$(\tilde A, \tilde B) = \big(e^{A \Delta},\; A^{-1}(e^{A \Delta}-I)B \big)$ in the scalar case.

If we truncate the exponential at first order,
\begin{equation}
e^{a\Delta} \approx 1 + a \Delta, \qquad
\frac{e^{a\Delta}-1}{a} \approx \Delta,
\end{equation}
we recover the Euler pair $(I+\Delta A,\, \Delta B)$.

If we Let $u(t) \equiv 1$ on $[t_k, t_{k+1})$ with arbitrary $h_k$.
\begin{align*}
h_{k+1}^{\text{ZOH}} &= e^{-\Delta} h_k + 2(1-e^{-\Delta}), \\
h_{k+1}^{\text{Euler}} &= (1-\Delta)h_k + 2\Delta.
\end{align*} 
Then Using Talyr expansion, we obtain
\begin{equation}
\underbrace{h_{k+1}^{\text{ZOH}}}_{\text{exact}}
= \Big(1 - \Delta + \tfrac{\Delta^2}{2} - \cdots\Big)h_k
+ 2\Big(\Delta - \tfrac{\Delta^2}{2} + \tfrac{\Delta^3}{6} - \cdots\Big).
\end{equation}
Euler keeps only the first-order pieces: $h_{k+1}^{\text{Euler}} = (1-\Delta)h_k + 2\Delta$. Then the per-step difference is
$h_{k+1}^{\text{ZOH}} - h_{k+1}^{\text{Euler}}
= \tfrac{\Delta^2}{2}\,h_k - \Delta^2 + O(\Delta^3)
= O(\Delta^2)$.

\paragraph{Numerical illustration (from $h_k=0$).}
\begin{itemize}
\item $\Delta=0.5$: \quad
$h_{k+1}^{\text{ZOH}} = 2(1-e^{-0.5}) \approx 0.78694$, \;
$h_{k+1}^{\text{Euler}} = 1.0$, \;
error $\approx 0.21306$.
\item $\Delta=0.25$: \quad
$h_{k+1}^{\text{ZOH}} \approx 0.44240$, \;
$h_{k+1}^{\text{Euler}} = 0.5$, \;
error $\approx 0.05760$.
\end{itemize}

As $\Delta$ halves, the error shrinks by a factor of $\approx 4$ (consistent with $O(\Delta^2)$ local error), while ZOH is exact under the constant-input assumption.

\section*{Appendix C. More Time Analysis Results }
Consistent with the main section, we provide additional long-time stability analysis on a different dataset with an alternative baseline. Figure~\ref{fig:fu3} illustrates the long-time integration stability of OFormer and Galerkin Transformer, both with and without the incorporation of our proposed MNO. The OFormer baseline demonstrates a noticeably faster convergence rate during training compared to our proposed model. This behavior can be attributed to OFormer’s normalization-based operator formulation, which effectively stabilizes the optimization dynamics by rescaling both the input and operator spectra. Such normalization reduces gradient variance and facilitates smoother propagation through layers, thereby accelerating convergence. Although OFormer exhibits a steeper early decrease in relative L2 error after the initial spike (faster short‑horizon convergence), our proposed Mamba neural operator suppress that spike and maintain lower error growth over the rollout. For the OFormer backbone the asymptotic errors are comparable, but MNO avoids the large early overshoot and yields a smoother, more stable trajectory. In both cases, the baseline models (blue curves) exhibit rapid error growth during extended rollout, with the error plateauing at a relatively high level or even drifting upward. By contrast, when leveraging Mamba neural operator (orange curves), the models maintain substantially lower error throughout the entire temporal horizon. The error remains bounded and exhibits minimal drift, indicating that MNO effectively mitigates the accumulation of integration errors over long rollouts. This demonstrates that our method not only improves short-term predictive accuracy but also provides enhanced numerical stability and robustness for long-time forecasting of dynamical systems.

\begin{table}[t!]
\centering
\caption{Comprehensive notation reference for symbols used throughout the manuscript. The table provides quick lookup for mathematical definitions and aligns notation across different sections.}
\label{tab:notation}
\begin{tabular}{ll}
\hline
\textbf{Notation} & \textbf{Description} \\
\hline
$h(t) \in \mathbb{R}^N$ & State vector at continuous time $t$  \\
$u(t) \in \mathbb{R}^L$ & Control/input vector  \\
$y(t) \in \mathbb{R}^L$ & Output vector \\
$A \in \mathbb{C}^{N \times N}$ & State transition matrix \\
$B \in \mathbb{C}^{N}$ & Input-to-state matrix  \\
$C \in \mathbb{C}^{N}$ & State-to-output matrix \\
$D \in \mathbb{C}^{N}$ & Direct input-to-output matrix \\
$\Delta$ & Time discretisation step size \\
$x_k \approx x(t_k)$ & State at discrete time index $k$, with $t_k = k \Delta$ \\
$u_k \approx u(t_k)$ & Input at discrete time index $k$ \\
$y_k \approx y(t_k)$ & Output at discrete time index $k$ \\
$A_d = e^{A \Delta}$ & Discretised state transition matrix (from matrix exponential) \\
$B_d = \int_{0}^{\Delta} e^{A \tau} B \, d\tau$ & Discretised input matrix \\
$x_{k+1} = A_d x_k + B_d u_k$ & Discrete-time state update equation \\
$y_k = C x_k + D u_k$ & Discrete-time observation equation \\
\hline
Hyperparameters \\
$d\_state$ & size of the SSM hidden state\\
$ssm\_ratio$ & share of $d_{\text{model}}$ channels in the SSM branch\\
\hline
\end{tabular}
\end{table}
\begin{figure*}[ht]
    \centering
    \includegraphics[width=\linewidth]{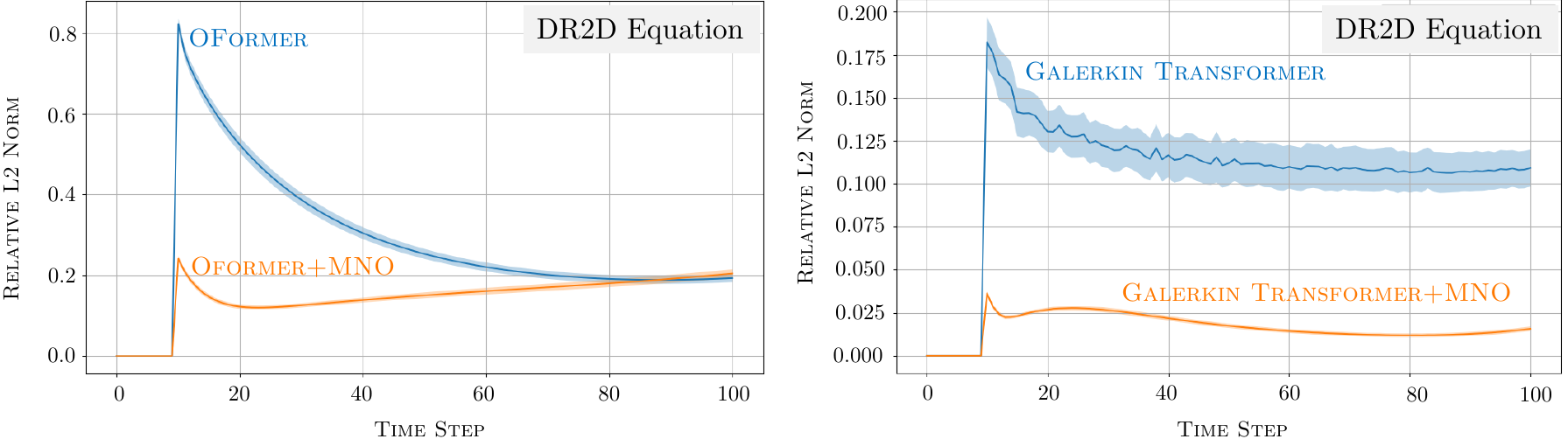}
    \caption{
    Relative L2 error on DR2D across time steps, comparing the performance of our Mamba Neural Operator (MNO) with the OFormer baseline and the Galerkin Transformer (G.T.). The left plot shows the evolution of relative L2 error over time steps for the DR2D equation, while the right plot shows the same for the DR2D equation. MNO consistently outperforms OFormer and G.T., particularly in terms of lower error growth, highlighting its superior ability to capture long-range dependencies and model complex dynamics more effectively.
    }
    \label{fig:fu3}
\end{figure*}

\clearpage 

\bibliographystyle{model1-num-names}

\bibliography{main}


\end{document}